\documentclass{article}

\usepackage{comment}
\excludecomment{draft}
\excludecomment{nextversion}

\usepackage[utf8]{inputenc} %
\usepackage[T1]{fontenc}    %
\usepackage{hyperref}       %
\usepackage{booktabs}       %

\usepackage{natbib}
\usepackage{babel}

\usepackage{amsthm}
\usepackage{amssymb}
\usepackage{amsfonts} %
\usepackage{thmtools,thm-restate}
\usepackage{mathtools} %
\usepackage{mathcommand}
\usepackage{nicefrac}       %
\usepackage{etoolbox}
\usepackage[bb=dsserif]{mathalpha}
\usepackage{cleveref}

\usepackage[dvipsnames]{xcolor}
\usepackage{tikz}
\newcommand*\Circled[1]{\tikz[baseline=(char.base)]{
		\node[shape=circle,draw,inner sep=2pt] (char) {#1};}}

\usepackage{pifont}
\usepackage{graphicx}
\usepackage{booktabs} %
\usepackage{fancyhdr}
\usepackage{multirow}

\usepackage{url}            %
\usepackage{microtype}
\usepackage{graphicx}
\usepackage{subcaption} 
\usepackage{placeins} %

\usepackage{enumitem}
\setitemize{noitemsep,topsep=0pt,parsep=0pt,partopsep=0pt}
\setenumerate{noitemsep,topsep=0pt,parsep=0pt,partopsep=0pt}

\graphicspath{{./figs/}}

\newtheorem{theorem}{Theorem}[section]

\newtheorem{definition}[theorem]{Definition}
\newtheorem{corollary}[theorem]{Corollary}

\newtheorem{proposition}[theorem]{Proposition}

\DeclareMathOperator{\opExpectation}{\mathbb{E}}
\newcommand{\E}[2]{\opExpectation_{#1} \left [ \ifblank{#2}{\:\cdot\:}{#2} \right ]}
\newcommand{\simpleE}[1]{\opExpectation_{#1}} %

\providecommand\given{\MidSymbol[\vert]}

\newcommand\MidSymbol[1][]{%
\nonscript\:#1
\allowbreak
\nonscript\:
\mathopen{}}

\DeclareMathOperator{\opInformationContent}{h}
\DeclarePairedDelimiterXPP{\ICof}[1]{\opInformationContent}{(}{)}{}{%
    \ifblank{#1}{\:\cdot\:}{#1}
}

\DeclareMathOperator{\opEntropy}{H}
\DeclarePairedDelimiterXPP{\Hof}[1]{\opEntropy}{[}{]}{}{%
    \renewcommand\given{\MidSymbol[\delimsize\vert]}
    \ifblank{#1}{\:\cdot\:}{#1}
}
\DeclarePairedDelimiterXPP{\xHof}[1]{\opEntropy}{(}{)}{}{%
    \ifblank{#1}{\:\cdot\:}{#1}
}

\DeclareMathOperator{\opMI}{I}
\DeclarePairedDelimiterXPP{\MIof}[1]{\opMI}{[}{]}{}{%
    \renewcommand\given{\MidSymbol[\delimsize\vert]}
    \ifblank{#1}{\:\cdot\:}{#1}
}

\DeclarePairedDelimiterXPP{\CrossEntropy}[2]{\opEntropy}{(}{)}{}{%
    \ifblank{#1#2}{\:\cdot\: \MidSymbol[\delimsize\Vert] \:\cdot\:}{#1 \MidSymbol[\delimsize\Vert] #2}
}

\DeclareMathOperator{\opKale}{D_\mathrm{KL}}
\DeclarePairedDelimiterXPP{\Kale}[2]{\opKale}{(}{)}{}{%
    \ifblank{#1#2}{\:\cdot\: \MidSymbol[\delimsize\Vert] \:\cdot\:}{#1 \MidSymbol[\delimsize\Vert] #2}
}

\DeclareMathOperator{\opp}{p}
\DeclarePairedDelimiterXPP{\pof}[1]{\opp}{(}{)}{}{%
    \renewcommand\given{\MidSymbol[\delimsize\vert]}
    \ifblank{#1}{\:\cdot\:}{#1}
}

\DeclarePairedDelimiterXPP{\pcof}[2]{\opp_{#1}}{(}{)}{}{%
    \renewcommand\given{\MidSymbol[\delimsize\vert]}
    \ifblank{#2}{\:\cdot\:}{#2}
}

\DeclarePairedDelimiterXPP{\hpcof}[2]{\hat{\opp}_{#1}}{(}{)}{}{%
    \renewcommand\given{\MidSymbol[\delimsize\vert]}
    \ifblank{#2}{\:\cdot\:}{#2}
}

\DeclareMathOperator{\opq}{q}
\DeclarePairedDelimiterXPP{\qof}[1]{\opq}{(}{)}{}{%
    \renewcommand\given{\MidSymbol[\delimsize\vert]}
    \ifblank{#1}{\:\cdot\:}{#1}
}

\DeclarePairedDelimiterXPP{\qcof}[2]{\opq_{#1}}{(}{)}{}{%
    \renewcommand\given{\MidSymbol[\delimsize\vert]}
    \ifblank{#2}{\:\cdot\:}{#2}
}

\DeclarePairedDelimiterXPP{\varHof}[2]{\opEntropy_{\ifblank{#1}{\:\cdot\:}{#1}}}{[}{]}{}{%
    \renewcommand\given{\MidSymbol[\delimsize\vert]}
    \ifblank{#2}{\:\cdot\:}{#2}
}

\DeclarePairedDelimiterXPP{\xvarHof}[2]{\opEntropy_{\ifblank{#1}{\:\cdot\:}{#1}}}{(}{)}{}{%
    \renewcommand\given{\MidSymbol[\delimsize\vert]}
    \ifblank{#2}{\:\cdot\:}{#2}
}

\newcommand{\Dtrain}{{\mathcal{D}^\text{train}}}

\newcommand{\Dpool}{{\mathcal{D}^\text{pool}}}

\newcommand{\w}{\omega}
\newcommand{\W}{\Omega}

\newcommand{\xeval}{x^\text{eval}}
\newcommand{\xtest}{x^\text{test}}
\newcommand{\xtrain}{x^\text{train}}
\newcommand{\xbatch}{x^\text{acq}}
\newcommand{\xpool}{x^\text{pool}}

\newcommand{\Xeval}{X^\text{eval}}
\newcommand{\Xtest}{X^\text{test}}
\newcommand{\Xtrain}{X^\text{train}}
\newcommand{\Xbatch}{X^\text{acq}}
\newcommand{\Xpool}{X^\text{pool}}

\newcommand{\yeval}{y^\text{eval}}
\newcommand{\ytest}{y^\text{test}}
\newcommand{\ytrain}{y^\text{train}}
\newcommand{\ybatch}{y^\text{acq}}
\newcommand{\ypool}{y^\text{pool}}

\newcommand{\Ytest}{Y^\text{test}}
\newcommand{\Ytrain}{Y^\text{train}}
\newcommand{\Yeval}{Y^\text{eval}}
\newcommand{\Ybatch}{Y^\text{acq}}
\newcommand{\Ypool}{Y^\text{pool}}

\newcommand{\xevalset}{\{x^\text{eval}_i\}_i}
\newcommand{\xtestset}{\{x^\text{test}_i\}_i}
\newcommand{\xtrainset}{\{x^\text{train}_i\}_i}
\newcommand{\xbatchset}{\{x^\text{acq}_i\}_i}
\newcommand{\xpoolset}{\{x^\text{pool}_i\}_i}

\newcommand{\Xevalset}{\{X^\text{eval}_i\}_i}
\newcommand{\Xtestset}{\{X^\text{test}_i\}_i}
\newcommand{\Xtrainset}{\{X^\text{train}_i\}_i}
\newcommand{\Xbatchset}{\{X^\text{acq}_i\}_i}
\newcommand{\Xpoolset}{\{X^\text{pool}_i\}_i}

\newcommand{\yevalset}{\{y^\text{eval}_i\}_i}
\newcommand{\ytestset}{\{y^\text{test}_i\}_i}
\newcommand{\ytrainset}{\{y^\text{train}_i\}_i}
\newcommand{\ybatchset}{\{y^\text{acq}_i\}_i}
\newcommand{\ypoolset}{\{y^\text{pool}_i\}_i}

\newcommand{\Ytestset}{\{Y^\text{test}_i\}_i}
\newcommand{\Ytrainset}{\{Y^\text{train}_i\}_i}
\newcommand{\Yevalset}{\{Y^\text{eval}_i\}_i}
\newcommand{\Ybatchset}{\{Y^\text{acq}_i\}_i}
\newcommand{\Ypoolset}{\{Y^\text{pool}_i\}_i}

\newcommand{\xset}{\{x_i\}_i}

\newcommand{\yset}{\{y_i\}_i}

\newcommand{\xtrainsetfull}{\{x^\text{train}_i\}_{i \in \{1, \dots, |\Dtrain|\}}}
\newcommand{\xpoolsetfull}{\{x^\text{pool}_i\}_{i \in \{1, \dots, |\Dpool|\}}}

\newcommand{\ytrainsetfull}{\{y^\text{train}_i\}_{i \in \{1, \dots, |\Dtrain|\}}}

\newcommand{\xsetfull}{\{x_i\}_{i \in I}}

\newcommand{\ysetfull}{\{y_i\}_{i \in I}}

\newcommand{\N}{\mathcal{N}}
\newcommand{\normaldist}[2]{\N(#1,\,#2)}

\DeclareMathOperator*{\argmax}{arg\,max}

\newcommand{\defeq}{\vcentcolon=}

\newcommand{\andreas}[1]{}
\newcommand{\yarin}[1]{}
\newcommand{\editor}[1]{}

\begin{draft}
\renewcommand{\andreas}[1]{{\leavevmode\color{blue}{\footnotesize AK} {\tiny says: }#1}}
\renewcommand{\yarin}[1]{{\leavevmode\color{purple}{\footnotesize YG}{\tiny says}[[#1]]}}
\renewcommand{\editor}[1]{{\leavevmode\color{blue}{\footnotesize Editor} {\tiny says: }#1}}
\end{draft}

\usepackage[accepted]{icml2021_arxiv}

\everypar{\looseness=-1}
\allowdisplaybreaks

\newcommand{\mytitle}{A Practical \& Unified Notation for Information-Theoretic Quantities in ML}

\icmltitlerunning{\mytitle}

\begin{document}

\twocolumn[
\icmltitle{\mytitle}

\icmlsetsymbol{equal}{*}

\begin{icmlauthorlist}
\icmlauthor{Andreas Kirsch}{oatml}
\icmlauthor{Yarin Gal}{oatml}
\end{icmlauthorlist}

\icmlaffiliation{oatml}{OATML, University of Oxford}

\icmlcorrespondingauthor{Andreas Kirsch}{andreas.kirsch@cs.ox.ac.uk}

\icmlkeywords{Machine Learning, ICML}

\vskip 0.3in
]

\printAffiliationsAndNotice{}  %

\begin{abstract}
A practical notation can convey valuable intuitions and concisely express new ideas.
Information theory is of importance to machine learning, but the notation for information-theoretic quantities is sometimes opaque.
We propose a practical and unified notation and extend it to include information-theoretic quantities between observed outcomes (events) and random variables. This includes the point-wise mutual information known in NLP and mixed quantities such as specific surprise and specific information in the cognitive sciences and information gain in Bayesian optimal experimental design. 
We apply our notation to prove a version of Stirling's approximation for binomial coefficients mentioned by \citet{mackay2003information} using new intuitions. 
\yarin{you can comment here on unifying also w NLP's pointwise MI and extending (?) to mixed pointwise MI also instead of "we extend" you can use language like "we bridge IT developments in cognitive sciences and IT applications in ML"}%
We also concisely rederive the evidence lower bound for variational auto-encoders and variational inference in approximate Bayesian neural networks.
Furthermore, we apply the notation to a popular information-theoretic acquisition function in Bayesian active learning which selects the most informative (unlabelled) samples to be labelled by an expert and extend this acquisition function to the core-set problem with the goal of selecting the most informative samples \emph{given} the labels.
\end{abstract}

\yarin{"Bridging IT across cognitive sciences, NLP, and ML"?}

\section{Introduction}
Information theory has provided insights for deep learning: information bottlenecks explain objectives both for supervised and unsupervised learning of high-dimensional data \citep{shwartz2017opening,kirsch2020unpacking,jonsson2020convergence};\yarin{add citations to all} similarly, information theory has inspired Bayesian experiment design, Bayesian optimization, and active learning as well as provided inspiration for research into submodularity in general \citep{lindley1956measure, foster2019variational, kirsch2019batchbald}.
\andreas{add citations}

\yarin{perhaps start the para with:
But IT can be opaque to ML researchers, leading to misunderstandings and misuse [give examples]
Then say that a practical notation can accelerate research, here we propose a new practical notation, we give examples in a bunch of domains }
A practical notation conveys valuable intuitions and concisely expresses new ideas.
The currently employed notation in information theory, however, can be ambiguous for more complex expressions found in applied settings and often deviates\yarin{is this related to "unified" contribution?
perhaps add at the start of the para (problem setting) 
"Further, many papers use specialised notation which makes it difficult to draw connections between contributions"
and give examples for that in your chapters below as well} between published works because
researchers are from different backgrounds such as statistics, computer science, information engineering, which all use information theory.
For example, $H(X,Y)$ is sometimes used to denote the \emph{cross-entropy} between $X$ and $Y$, which conflicts with common notation of the joint entropy $H(X, Y)$ for $X$ and $Y$, or it is not clarified that $\Hof{X \given Y}$ as conditional entropy of $X$ given $Y$\yarin{instead of "not clarified" say that same authors use this to refer to an exp over Y, while others use the same expression to denote conditionining on a realisation [again, give examples in text below]

also, if your contribution is "unifying", I'd expect to see you discussing multiple existing high profile papers in the field, highlight contradicting notation, and recast them in your unified notation (and give insight for what we gained from this)} is an expectation over $Y$.
We present a disambiguated and consistent notation while striving to stay close to known notation when possible. 

In addition, we show that an extension of information-theoretic quantities to relations between observed outcomes (events) and random variables can be of great use in machine learning. 
Commonly, the mutual information $\MIof{X; Y}$ is only defined for random variables $X, Y$, while in natural language processing the point-wise information \citep{church1990word} has been introduced for two outcomes. This follows earlier work in information theory by \citet{fano1962transmission}, which also considers a natural extension to the mutual information between an outcome $x$ and a random variable $Y$, referred to as `\emph{conditional average of the [point-wise] mutual information}'. 
Variants of this have been used more recently in the cognitive sciences and neuroscience as `\emph{(response\nobreakdash-)\allowbreak{}specific information}' and `\emph{specific surprise}' \citep{deweese1999measure, williams2011information}, but they might not be well-known outside of neuroscience and the cognitive sciences. Our consistent extension also unifies these two previously separate quantities.\yarin{are you further extending these two, or are you bringing this into ML?}
\andreas{note that deweese1999measure does all the hard lifting already! specific information = information gain and specific surprise = what we call information-theoretic surprise!}

As an application for information quantities on observed outcomes, we present a different and intuitive derivation of Stirling's approximation for binomial coefficients. The original deduction is found in \citet{mackay2003information} on page 2. 
Furthermore, we show this allows for a simple analysis of the approximation error.

As another application of the notation,\yarin{you can ask Tom for feedback on that section} we derive the evidence lower bound (ELBO) from \citet{kingma2014autoencoding} in a single (relatively long) line.

And, finally, as an application for mutual information terms that include observed outcomes, we examine the core-set problem which consists of selecting the most informative samples of a training set given the labels and provide new results. We also rederive\yarin{move this up to the VAE para?} the evidence-lower-bound inequality for variational inference of approximate Bayesian neural networks using our proposed notation.

The goal of this is to illustrate that our proposed notation is useful and show\yarin{allows us to easily take ideas from one application domain and extend them to solve new problems.} that it allows for more concise expression of important ideas.

Concretely, for the last example, we examine BALD (Bayesian Active Learning by Disagreement), an acquisition function in Bayesian active learning \citep{gal2017deep,houlsby2011bayesian}, and extend it to the core-set problem.
In \emph{pool}-based active learning, we have access to a huge reservoir of unlabelled data in a \emph{pool set} and iteratively select samples from this pool set to be labeled by an \emph{oracle} (e.g.\ human experts) to increase the model performance as quickly as possible. \emph{Acquisition functions} are used to score all pool samples and the highest scorer is acquired. The goal of an acquisition function is to score the most ``informative'' samples the highest. BALD maximizes the expected information gain $\MIof{\W; Y \given x}$ of the model parameters $\W$ given the prediction variable $Y$ for a candidate sample $x$ from the pool set.
It is equivalent to the concept of reduction in posterior uncertainty known from Bayesian optimal experimental design \citep{lindley1956measure}.
The \emph{core-set problem} on the other hand consists of identifying the most informative samples \emph{given} the labels, the \emph{core set}, such that training a model on this core set will perform as well as a model trained on the whole dataset.\yarin{btw Andreas Krause has a low of work on coresets

eg the stuff Mario Lucic worked on}
We examine the connection between BALD and information gain in a case where the information gain equals the information-theoretic surprise, which we define later. As such, we introduce \emph{Core-Set by Disagreement (CSD)}, which maximizes the information gain of the model parameters given the true label $y$ of a sample $x$ in the dataset.

\andreas{can we check with Krause et al whether infogain maximization is not new? it seems trivial but not sure people have actually applied it because it is not easy to compute}

\section{A Practical \& Unified Notation}
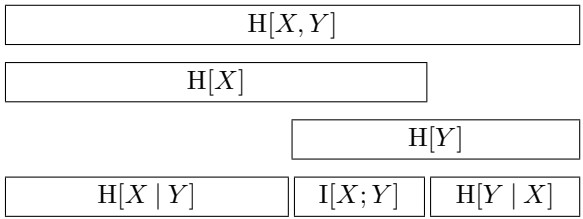
\begin{figure}[t]
  \begin{center}
  \setlength{\unitlength}{1in}
  \begin{picture}(3,1.10)(0,-0.2)
  \put(0,0.7){\framebox(3,0.20){\(\Hof{X,Y}\)}}
  \put(0,0.4){\framebox(2.2,0.20){\(\Hof{X}\)}}
  \put(1.5,0.1){\framebox(1.5,0.20){\(\Hof{Y}\)}}
  \put(1.5125,-0.2){\framebox(0.675,0.20){\(\MIof{X;Y}\)}}
  \put(0,-0.2){\framebox(1.475,0.20){\(\Hof{X \given Y}\)}}
  \put(2.225,-0.2){\framebox(0.775,0.20){\(\Hof{Y \given X}\)}}
  \end{picture}
  \end{center}
  \caption{\emph{Reproduction of Figure 8.1 from \citet{mackay2003information} using the new suggested notation:} The relationship between joint entropy, 
   marginal entropy, conditional entropy and mutual information.
  \label{fig:mackay_81}
  }%
\end{figure}

\yarin{this can be used for the BG section in the thesis / journal paper

split this section into what's used to date, and what you're proposing; reading this, I assume that all the stuff that you discuss is already existing lit. Start from a proper bg sect, explain all the different notations used atm and the abuse of notation and source of confusion. Then have your own contribution section which explains what you're suggesting to do differently (which can be brief) followed by all the stuff we can deduce from the unified notation }
For a general introduction to information theory, we refer to \citet{CTJ, yeung2008information}.\andreas{cite Olah's blog post too} In the following section, we introduce our practical and unified notation. We start with notation that is explicit about the probability distribution $\pof{}$.
\begin{definition}
  Let Shannon's information content $\ICof{}$, cross-entropy $\CrossEntropy{}{}$, entropy $\xHof{}$, and KL divergence $\Kale{}{}$ (Kullback-Leibler divergence) be defined for a probability distribution $\opp$ and non-negative function $\opq$ for a random variable \(X\) and non-negative real number \(\rho\) as:
  \begin{align}
      \ICof{\rho} &\defeq -\ln{\rho} \\
      \CrossEntropy{\pof{X}}{\qof{X}} &\defeq \simpleE{\pof{x}}{\ICof{\qof{x}}}\\
      \xHof{\pof{X}} &\defeq \CrossEntropy{\pof{X}}{\pof{X}} \\
      \Kale{\pof{X}}{\qof{X}} &\defeq \CrossEntropy{\pof{X}}{\qof{X}} - \xHof{\pof{X}}. 
  \end{align}
\end{definition}
\citet{shannon1948mathematical} introduced\yarin{as well as bunch of other properties - the justification is actually really nice starting from desiderata for a quantity of information and then proving that this log def is the \emph{only} one that satisfies the desiderata. You can comment on this given that most people didn't read the original paper from 48} the information content as negative logarithm due to its additivity for independent messages: \(\ICof{\pof{x,y}} = \ICof{\pof{x}} + \ICof{\pof{x,y}}\) for independent random variables \(X\) and \(Y\).
\begin{proposition}
  \label{prop:xe_kale_rules}
  For a random variable \(X\) with probability distributions \(\opp\), \(\opp_1\) and \(\opp_2\), and non-negative functions \(\opq\), \(\opq_1\) and \(\opq_2\) and \(\alpha \in [0,1]\):
  \begin{align}
    &\CrossEntropy{\opp}{\alpha \opq} = \CrossEntropy{\opp}{\opq} + \ICof{\alpha}, \\
    &\CrossEntropy{\opp}{\opq^\alpha} = \alpha \CrossEntropy{\opp}{\opq} \\
    &\CrossEntropy{\opp}{\opq_1 \opq_2} = \CrossEntropy{\opp}{ \opq_1} + \CrossEntropy{\opp}{ \opq_2}, \\
    &\CrossEntropy{\alpha \opp_1 + (1 - \alpha) \opp_2}{\opq} = \\
    & \quad = \alpha \CrossEntropy{\opp_1}{\opq} + (1 - \alpha) \CrossEntropy{\opp_2}{\opq} \\
    & \quad = \CrossEntropy{\opp_1}{\opq^\alpha } + \CrossEntropy{\opp_2}{\opq^{1 - \alpha}},
  \end{align}
  where we have left out ``\((X)\)'' everywhere for brevity.
\end{proposition}
\begin{proof}
  The statements follows from the linearity of the expectation and the additivity of the logarithm for products.
\end{proof}
This can be extended to show that cross-entropies are linear in their left-hand argument and log-linear in their right-hand argument.

When we want to emphasize that we approximate the true distribution \(\opp\) using a different distribution \(\opq\) and the true probability distribution $\opp$ is understood, we use the notation $\varHof{\opq}{}$ for $\CrossEntropy{\pof{}}{\qof{}}$ following notation in \citet{kirsch2020unpacking, xu2020theory}:
\begin{definition}
When the true probability distribution $\opp$ is understood from context, we will use the following short-hands:
\begin{align}
    \Hof{X} &\defeq \xHof{\pof{X}} \\
    \varHof{\opq}{X} &\defeq \CrossEntropy{\pof{X}}{\qof{X}}.
\end{align}
When we have a parameterized distribution $\opq_\theta$ with parameters $\theta$, we will write $\varHof{\theta}{}$ instead of $\varHof{\opq_\theta}{}$ when the context is clear. 
\end{definition}
Approximating a possibly intractable distribution with a parameterized one is common when performing variational inference, that is.\andreas{forward cite the sections that do this?}
The main motivation for this notation is that when $\opq$ is the density of a distribution, that is $\int q(x) \, dx = 1$, we have $\varHof{\opq}{} \ge \Hof{}$ with equality when \(\opq = \opp\).
Thus, we have the following useful identities:
\begin{proposition}
  \yarin{give citations for all prev known results (even if it's just to a book)}We have the following lower-bounds\yarin{do you want perhaps to present these as "useful identities" that you use throughout? I'm confused about the context of why this is presented} for the cross-entropy and KL, with \(Z_q \defeq \int q(x) \, dx\):
  \begin{align}
    \CrossEntropy{\pof{X}}{\qof{X}} &\ge \xHof{\pof{X}} + \ICof{Z_q}, \\
    \Kale{\pof{X}}{\qof{X}} &\ge \ICof{Z_q},
  \end{align}
  with equality exactly when when \(\opq/Z_q = \opp\) for \(Z_q \defeq \int q(x) \, dx\).
\end{proposition}
\begin{proof}
  The statements follow from Jensen's inequality and the convexity of \(\ICof{}\).
\end{proof}
This also implies the non-negativity of the KL for densities when we substitute $Z_q=1$ in above statements. We repeat the result as it is often used: 
\begin{corollary}
When $q$ is a probability distribution, we have:
  \begin{align}
    \CrossEntropy{\pof{X}}{\qof{X}} &\ge \xHof{\pof{X}}, \label{eq:cross_entropy_vs_entropy} \\
    \Kale{\pof{X}}{\qof{X}} &\ge 0,
  \end{align}  
  with equality exactly when $\opq = \opp$.\andreas{equality under distribution, e.g. for almost all}
\end{corollary}
Note that for continuous distributions, above equality \(\opp = \opq\) only has to hold almost everywhere.

Above definitions are trivially extended to joints of random variables by substituting the random variable of the product space. Similarly, the conditional entropy is defined by taking the expectation over both \(X\) and \(Y\).
For example:
\begin{proposition}
  Given random variables \(X\) and \(Y\), we have:
  \begin{align}
    \Hof{X, Y} &= \simpleE{\pof{x, y}}{\ICof{\pof{x, y}}}; \\
    \Hof{X \given Y} &= \simpleE{\pof{x, y}}{\ICof{\pof{x \given y}}}.
  \end{align}
\end{proposition}
In particular, note that \(\Hof{X \given Y}\) is an expectation over \(X\) \emph{and} \(Y\).

For cross-entropies and KL divergences, we expand the definitions similarly. In particular, we have the following equality for cross-entropies, which follows from these definitions:
\begin{align}
  &\CrossEntropy{\pof{X \given Y}}{\qof{X \given Y}} = \simpleE{\pof{x,y}}{\ICof{\qof{x \given y}}} \notag \\
  & \quad = \CrossEntropy{\pof{X, Y}}{\qof{X \given Y}}.
\end{align}
The last idiosyncrasy only applies to cross-entropies. Indeed, for KL divergences, we have:
\begin{align}
  &\Kale{\pof{X \given Y}}{\qof{X \given Y}} = \\
  & \quad = \CrossEntropy{\pof{X, Y}}{\qof{X \given Y}} - \xHof{\pof{X \given Y}} \\
  &\Kale{\pof{X, Y}}{\qof{X \given Y}} = \\
  & \quad = \CrossEntropy{\pof{X, Y}}{\qof{X \given Y}} - \xHof{\pof{X, Y}}.
\end{align}
Note, that the second terms are usually not equal \(\xHof{\pof{X \given Y}} \not= \xHof{\pof{X, Y}}\), and the two terms are thus different.

The reader might wonder when we are interested in \(\Kale{\pof{X, Y}}{\qof{X \given Y}}\). It can arise when performing symbolic manipulations, so we mention it explicitly here.

\citet{mackay2003information} has an elegant visualizations for information quantites, which we reproduce in \Cref{fig:mackay_81}. \citet{yeung1991outlook} introduces I-diagrams which provide another useful intuitive approach, but they do not scale as easily to what we introduce next.

\yarin{Put this as its own section to distinguish from BG}
\textbf{Observed outcomes.} So far, we have introduced well-known information-theoretic quantities using a more consistent notation.\yarin{I'm confused if prev text is claimed to be a contribution, or just organising the lit into consistent notation that we can use for the next section which is the contribution. If prev text was a contributoin, I'd expect to see a BG sect with a discussion of existing notation and why it's not sufficient}
Now, we further canonically extend\yarin{discuss existing PMI in NLP first, what it's used for, why it's important, why people don't use it in ML. Similarly for mixed PMI from cognitive science 

then re-pitch this as "we bridge the above two into ML" rather than "we extend"} the definitions to tie random variables to specific observed outcomes, e.g.\ $X = x$. We refer to $X$ when we have $X=x$ in an expression as \emph{tied random variable} as it is \emph{tied} to an outcome. If we mix \emph{(untied) random variables} and \emph{tied random variables}, we define $\Hof{}$ as an operator which takes an expectation of Shannon's information content for the given expression over the (untied) random variables conditioned on the tied outcomes. For example, $\Hof{X, Y=y \given Z, W=w} = \simpleE{\pof{X, Z \given y, w}} \ICof{\pof{x, y \given z, w}}$ following this notation. We generally shorten \(Y=y\) to \(y\) when the connection is clear from context.
Similarly, we have $\CrossEntropy{\pof{X \given y}}{\qof{X \given y}} = \simpleE{\pof{x \given y}} \ICof{\qof{x \given y}}$.
Importantly, all of the above maintain the identities $\Hof{X, Y} = \simpleE{\pof{x}} \Hof{x, Y} = \simpleE{\pof{y}} \Hof{X, y}$, which is the motivation behind these extensions. \Cref{fig:mixed_entropy_diagram} provides an overview over the quantities for two random variables \(X\) and \(Y\) when \(Y=y\) is observed. We define everything in detail below and provide intuitions.

\begin{definition}
  Given random variables $X$ and $Y$ and outcome $y$, we define:
  \begin{align}
    \Hof{y} &\defeq \ICof{\pof{y}} \\
    \Hof{X, y} &\defeq \simpleE{\pof{x \given y}} \Hof{x, y} = \simpleE{\pof{x \given y}} \ICof{\pof{x, y}} \\
    \Hof{X \given y} &\defeq \simpleE{\pof{x \given y}} \Hof{x \given y} = \simpleE{\pof{x \given y}} \ICof{\pof{x \given y}} \\
    \Hof{y \given X} &\defeq \simpleE{\pof{x \given y}} \Hof{y \given x} = \simpleE{\pof{x \given y}} \ICof{\pof{y \given x}},
  \end{align}
  where we have shortened $Y=y$ to $y$.
\end{definition}
Note \(\Hof{y}, \Hof{X, y}\), and so on are shorthands for \(\xHof{\pof{y}}\), \(\xHof{\pof{X, y}}\), and so on. Shannon's information content could also be defined as a special case of \(\xHof{}\), as we have \(\ICof{\pof{x}} = \xHof{\pof{x}}\).

The intuition from information theory behind these definitions is that, e.g., $\Hof{X, y}$ measures the average length of transmitting $X$ and $Y$ together when $Y=y$ unbeknownst to the sender and receiver, and
\(\Hof{y \given X}\) measures how much additional information needs to be transferred on average for the receiver to learn \(y\) when it already knows $X \given y$.

As a memory hook for the reader, lower-case letters are always used for tied random variables and upper-case letters for (untied) random variables over which we take an expectation. This makes it easy to differentiate between the two cases and write down the actual expressions.

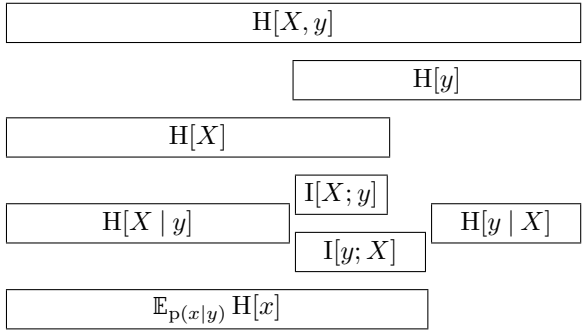
\begin{figure}
  \begin{center}
  \setlength{\unitlength}{1in}
  \begin{picture}(3,1.70)(0,-0.8)
  \put(0,0.7){\framebox(3,0.20){\(\Hof{X, y}\)}}
  \put(1.5,0.4){\framebox(1.5,0.20){\(\Hof{y}\)}}
  \put(0,0.1){\framebox(2.0,0.20){\(\Hof{X}\)}}
  \put(1.5125,-0.2){\framebox(0.475,0.20){\(\MIof{X ; y}\)}}
  \put(1.5125,-0.5){\framebox(0.675,0.20){\(\MIof{y ; X}\)}}
  \put(0,-0.35){\framebox(1.475,0.20){\(\Hof{X \given y}\)}}
  \put(2.225,-0.35){\framebox(0.775,0.20){\(\Hof{y\given X}\)}}
  \put(0,-0.8){\framebox(2.2,0.20){\(\simpleE{\pof{x \given y}}\Hof{x}\)}}
  \end{picture}
  
  \end{center}
  {%
  \caption{\emph{The relationship between between joint entropy \(\Hof{X, y}\), 
  entropies \(\Hof{X}, \Hof{y}\), conditional entropies \(\Hof{X \given y}, \Hof{y \given X}\), information gain \(\MIof{X ; y}\) and surprise \(\MIof{y ; X}\) when \(Y=y\) is observed.}
  We include \(\simpleE{\pof{x \given y}}\Hof{x}\) to visualize \Cref{proposition:observed_conditional_entropies}.
  The figure follows Figure 8.1 in \citet{mackay2003information}.}
  \label{fig:mixed_entropy_diagram}
  }%
\end{figure}

From above definition, we also have \(\Hof{x, y} = \ICof{\pof{x, y}}\) and \(\Hof{x \given y} = \ICof{\pof{x \given y}}\).
Beware, however, that while we have \(\Hof{X \given y} = \Hof{X, y} - \Hof{y}\), for \(\Hof{y \given X}\), there is no such equality for \(\Hof{y \given X}\):
\begin{proposition}
  \label{proposition:observed_conditional_entropies}
  Given random variables $X$ and $Y$ and outcome $y$, we generally have:
  \begin{align}
    \Hof{X \given y} &= \Hof{X, y} - \Hof{y} \\
    \Hof{y \given X} &= \Hof{X, y} - \simpleE{\pof{x \given y}} \Hof{x} \notag \\
    & \mathrel{\textcolor{red}{\not=}} \Hof{X, y} - \Hof{X},
  \end{align}
\end{proposition} 
\begin{proof}
  \(\Hof{X \given y} = \Hof{X, y} - \Hof{y}\) follows immediately from the definitions.
  \(\Hof{y \given X} \not=\Hof{X, y} - \Hof{X}\) follows because, generally, \(\simpleE{\pof{x \given y}} \Hof{x} \not= \Hof{X}\) when \(\pof{x \given y} \not= \pof{x}\). 
  E.g., for \(X\) and \(Y\) only taking binary values, $0$ or $1$, let\footnote{See also \url{https://colab.research.google.com/drive/1HvLXUMQYcxMGZ4S_a00xddGmfz0IHaR3}.} \(\pof{x, y}=\tfrac{1}{3} \mathbb{1}
_{\{x=0=y\}}\), then \(\simpleE{\pof{x \given y}} \Hof{x} = \log{\left(\frac{3}{2} \right)} \not= \log{\left(\frac{3 \sqrt[3]{2}}{2} \right)} = \Hof{X}\).
\end{proof}

The mutual information and point-wise mutual information \citep{fano1962transmission,church1990word}\yarin{PMI: give a proper explanation of what these works do in the BG section, then say that later (ie here) you'll tie their thing to your thing
} are defined as:
\begin{definition}
  For random variables $X$ and $Y$ and outcomes $x$ and $y$ respectively, the point-wise mutual information $\MIof{x; y}$ and the mutual information $\MIof{X; Y}$ are:
  \begin{align}
    & \MIof{x; y} \defeq \Hof{x} - \Hof{x \given y} = \ICof{\frac{\pof{x}\pof{y}}{\pof{x, y}}} \\
    & \MIof{X; Y} \defeq \Hof{X} - \Hof{X \given Y} = \simpleE{\pof{x, y} } \MIof{x; y}.
  \end{align}
\end{definition}
This is similarly extended to $\MIof{X ; Y \given Z} = \Hof{X \given Z} - \Hof{X \given Y, Z}$ or $\MIof{X_1, X_2 ; Y} = \Hof{X_1, X_2} - \Hof{X_1, X_2 \given Y}$ and so on.

There are two common, sensible quantities we can define when we want to consider the information overlap between an random variable and an outcome: the \emph{information gain}, also known as \emph{specific information} and the \emph{surprise} \citep{deweese1999measure, butts2003much}. These two quantities\yarin{give a proper explanation of what these works do in the BG section, then say that later (ie here) you'll tie their thing to your thing} are usually defined separately in the cognitive sciences and neuroscience \citep{williams2011information}; however, we can unify them after relaxing the symmetry of the mutual information as done above: %
\begin{definition}
  Given random variables $X$ and $Y$ and outcome $y$ for $Y$, we define\yarin{similarly to Shanon's work, instead of extending the def's because they look nice "let's plug this thing into the notation and see what we get", you can start from the desiderata of partial MI, and derive from that desiderata that your def is the only one that satisfies it} the \emph{information gain} $\MIof{X; y}$ and the \emph{surprise} $\MIof{y; X}$ as:
  \begin{align}
    & \MIof{X ; y} \defeq \Hof{X} - \Hof{X \given y} \\
    & \MIof{y ; X} \defeq \Hof{y} - \Hof{y \given X}.
  \end{align}
\end{definition}
This unifying definition is novel to the best of our knowledge. It works by breaking the symmetry that otherwise exists for the regular and point-wise mutual information.

Note that the surprise can also be expressed as \(\MIof{y ; X} = \Kale{\pof{X \given y}}{\pof{X}}\). For example, this is done in \citet{bellemare2016unifying}---even though the paper mistakenly calls this surprise an information gain when it is not. 

We enumerate a few equivalent ways of writing the mutual information and surprise---the information gain has no such equivalences. This can be helpful to spot these quantities in the wild. 
\begin{proposition}
  We have
  \begin{align}
    \MIof{X; Y} &= \Kale{\pof{X, Y}}{\pof{X} \pof{Y}} \\
    \MIof{y ; X} &= \simpleE{\pof{x \given y}} \MIof{y ; x} \\
                 &= \E{\pof{x \given y}}{\Hof{x}} - \Hof{X \given y} \\
                 &= \Kale{\pof{X \given y}}{\pof{X}}.
  \end{align}
\end{proposition}

The information gain \(\MIof{X ; y}\) for $X$ given $y$ measures the reduction in uncertainty about $\Hof{X}$ when we observe $y$. \(\Hof{X}\) is the uncertainty about the true \(X\) that we want to learn as then entropy quantifies the amount of additional information that we need to transmit to fix \(X\), and similarly \(\Hof{X \given y}\) quantities the additional information we need to transmit to fix \(X\) once $y$ is known \citep{lindley1956measure}.
On the other hand, the surprise \(\MIof{y ; X}\) of $y$ for $X$ measures\yarin{I would add more info in the BG about the info theory view as introduced by Shanon if you use that as intuition} 
how much the posterior $X \given y$ lies in areas where $\pof{x}$ was small before observing $y$ \citep{deweese1999measure}. 

An important difference between the two is that the information gain can be \emph{chained} while the surprise cannot: 
\begin{proposition}
  Given random variables $X$, $Y_1$, and \(Y_2\) and outcomes $y_1$ and $y_2$ for $Y_1$ and \(Y_2\), respectively, we\yarin{you need to introduce your notation "y1,y2" first (def'd as "{..}")} have:
  \begin{align}
    \MIof{X ; y_1, y_2} & = \MIof{X ; y_1} + \MIof{X ; y_2 \given y_1} \\
    \MIof{y_1, y_2 ; X} & \mathrel{\textcolor{red}{\not=}} \MIof{y_1 ; X} + \MIof{y_2 ; X \given y_1}.
  \end{align}   
\end{proposition}
\begin{proof}
We have
  \begin{align*}
    &\MIof{X ; y_1, y_2} = \Hof{X} - \Hof{X \given y_1, y_2} \\
    & \quad = \Hof{X} - \Hof{X \given y_1} + \Hof{X \given y_1} - \Hof{X \given y_1, y_2} \\ 
    & \quad = \MIof{X ; y_1} + \MIof{X ; y_2 \given y_1},
  \end{align*}
while
  \begin{align*}
    \MIof{y_1, y_2 ; X} &= \simpleE{\pof{x \given y_1, y_2}} \MIof{y_1, y_2; x} \\
    & = \underbrace{\simpleE{\pof{x \given y_1, y_2}} \MIof{y_1; x}}_{\quad \displaystyle\mathrel{\color{red}\not=}\simpleE{\pof{x \given y_1}} \MIof{y_1; x} = \MIof{y_1 ; X}} \\
    & \quad + \underbrace{\simpleE{\pof{x \given y_1, y_2}} \MIof{y_2 ; x \given y_1}}_{\displaystyle=\MIof{y_2; X \given y_1}}.
  \end{align*}
That is, generally, \(\simpleE{\pof{x \given y_1, y_2}} \MIof{y_1; x} \not= \MIof{y_1 ; X}\).
To conclude the proof, we instantiate \(\pof{x \given y_1, y_2} \not=\pof{x \given y_1}\): 
for \(X\), \(Y_1\), and \(Y_2\) taking binary values \(0, 1\) only, let \(\pof{y_1}=\tfrac{1}{2}, \, \pof{x, y_2 \given y_1=0}=\tfrac{1}{4}, \,\pof{x \given y_2=0, y_1=1} = \tfrac{1}{2}, \,\pof{x=0 \given y_2=1, y_1=1} = 1\). 
Then \(\simpleE{\pof{x \given y_1, y_2}} \MIof{y_1; x} = \log{\left(\frac{2 \sqrt{3} \sqrt[4]{5}}{5} \right)} \not= \log{\left(\frac{6}{5} \right)} = \MIof{y_1 ; X}\) for \(y_1=1, y_2=1\) as the reader can easily verify\footnote{See also \url{https://colab.research.google.com/drive/1gn6oQohRMqXKEhyCogiVDcx1VZFkShaQ}.}.
\end{proof}
However, both quantities do chain in their (untied) random variables:
\begin{proposition}
  Given random variables $X_1$, $X_2$, $Y$, and outcome $y$ for $Y$:
  \begin{align}
    & \MIof{X_1, X_2 ; y} = \MIof{X_1 ; y} + \MIof{X_2 ; y \given X_1} \\
    & \MIof{y; X_1, X_2 } = \MIof{y ; X_1} + \MIof{y ; X_2 \given X_1}.
  \end{align}
\end{proposition}
\begin{proof}
  We have
  \begin{align*}
    & \MIof{X_1 ; y} + \MIof{X_2 ; y \given X_1} = \\
    & \quad = \Hof{X_1} - \Hof{X_1 \given y} + \Hof{X_2 \given X_1} + \Hof{X_2 \given X_1, y} \\
    & \quad = \underbrace{\Hof{X_1} + \Hof{X_2 \given X_1}}_{=\Hof{X_1, X_2}} -
       (\underbrace{\Hof{X_1 \given y} + \Hof{X_2 \given X_1, y}}_{=\Hof{X_1, X_2 \given y}}) \\
    & \quad = \MIof{X_1, X_2 ; y}.
  \end{align*}
  Similarly, we have
  \begin{align*}
    & \MIof{y ; X_1} + \MIof{y ; X_2 \given X_1} = \\
    & \quad = \Hof{y} - \Hof{y \given X_1} + \Hof{y \given X_1} - \Hof{y \given X_1, X_2} \\
    & \quad = \Hof{y} - \Hof{y \given X_1, X_2} \\
    & \quad = \MIof{y ; X_1, X_2}.
  \end{align*}
\end{proof}

These extensions of the mutual information are canonical as they permute with taking expectations over tied variables to obtain the regular (untied) quantities:
\begin{proposition}
  For random variables $X$ and $Y$:
  \begin{align}
    \MIof{X ; Y} &= \simpleE{\pof{y}} \MIof{X ; y} = \simpleE{\pof{y}} \MIof{y ; X} \\
    &= \simpleE{\pof{x, y}} \MIof{x, y}.
  \end{align}
\end{proposition}
\begin{proof}
  Follows immediately from substituting the definitions.
\end{proof}
Likewise, when all random variables are tied to a specific outcome, the quantities behaves as expected:
\begin{proposition}
  For random variables $X$, \(Y\), \(Y_1\) and \(Y_2\):
  \begin{align}
    \MIof{X; Y} &= \MIof{Y; X}, \text{and} \\
    \MIof{x ; y} &= \MIof{y ; x}; \\
    \MIof{X; Y_1, Y_2} &= \MIof{X; Y_1} + \MIof{X; Y_1 \given Y_2}, \text{and} \\
    \MIof{x; y_1, y_2} &= \MIof{x; y_1} + \MIof{x; y_2 \given y_1}.
  \end{align}
\end{proposition}
\begin{proof}
  The only interesting equality is \(\MIof{x; y_1, y_2} = \MIof{x; y_1} + \MIof{x; y_2 \given y_1}\):
  \begin{align*}
    &\MIof{x; y_1} + \MIof{x; y_2 \given y_1} = \\
    &\quad = \ICof{\frac{\pof{x} \, \pof{y_1}}{\pof{x, y_1}} \, \frac{\pof{x, y_1} \, \pof{y_1, y_2} \, \pof{y_1}}{\pof{y_1} \, \pof{y_1} \pof{x, y_1, y_2}}} \\
    &\quad = \ICof{\frac{\pof{x} \, \pof{y_1, y_2}}{\pof{x, y_1, y_2}}} \\
    &\quad = \MIof{x; y_1, y_2}.
  \end{align*}
\end{proof}

We can extend this to triple mutual information terms by adopting the extension $\MIof{X; Y; Z} = \MIof{X ; Y} - \MIof{X ; Y \given Z }$ \citep{yeung2008information} for outcomes as well: $\MIof{X ; Y; z} = \MIof{X ; Y} - \MIof{X ; Y \given z}$, which also works for higher-order terms.\andreas{appendix again...}

Overall, for the reader, there will be little surprise when working with the fully point-wise information-theoretic quantities, that is, when all random variables are observed. But the mixed ones require more care. We refer the reader back to \Cref{fig:mixed_entropy_diagram} to recall the relationships which also provide intuitions for the inequalities we will examine next.

\textbf{Inequalities.} We review some well-known inequalities first:
\begin{proposition}
  For random variables \(X\) and \(Y\),\yarin{move to bg} we have:
  \begin{align}
    \MIof{X ; Y} &\ge 0 \\
    \Hof{X} &\ge \Hof{X \given Y}, \\
  \intertext{and if \(X\) is a discrete random variables, we also have:}
    \Hof{X} &\ge 0 \\
    \MIof{X ; Y} &\le \Hof{X}.
  \end{align}
\end{proposition}
\begin{proof}
  The first two statements follow from:
  \begin{align}
    \Hof{X} - \Hof{X \given Y} &= \MIof{X ; Y} \notag \\
    &= \Kale{\pof{X, Y}}{\pof{X}\pof{Y}} \notag \\
    &\ge 0.
  \end{align}
  The third statement follows from the monotony of the expectation and \(\pof{x} \le 1\) for all \(x\).
\end{proof}
Following \citet{kirsch2020unpacking}, if we assume that we add independent\yarin{".."} \emph{zero-entropy noise} \(\epsilon_0 \sim \normaldist{0, \frac{1}{2\pi e}}\) to continuous random variables as observation noise, we can also force their continuous entropy to be non-negative: we have \(\Hof{X + \epsilon_0} \ge 0\) and also \(\MIof{X + \epsilon_0; Y} \le \Hof{X+\epsilon_0}\) as \(\MIof{X + \epsilon_0; Y} = \Hof{X + \epsilon_0} - \Hof{X + \epsilon_0 \given Y}\) and \(\Hof{X + \epsilon_0 \given Y} \ge 0\), too. We say, we \emph{inject zero-entropy noise} when we assume that zero-entropy noise has already been added to a continuous random variable.
\begin{corollary}
  For continuous random variables \(X\) and \(Y\) where we inject zero-entropy noise into $X$, we have:
  \begin{align}
    \Hof{X} &\ge 0 \label{eq:zero-entropy-noise} \\
    \MIof{X ; Y} &\le \Hof{X}.
  \end{align}
\end{corollary}
For mixed outcomes we find similar inequalities:\andreas{mention independence}
\begin{proposition}
  For random variables \(X\) and \(Y\) with outcome \(y\), we have:
  \begin{align}
    \MIof{y ; X} &\ge 0 \\
    \Hof{y} &\ge \Hof{y \given X} \\    
    \simpleE{\pof{x \given y}{\Hof{x}}} &\ge \Hof{X \given y}, \\
    \intertext{and if \(Y\) is a discrete random variable (or we inject zero-entropy noise), we also have:}
      \Hof{y \given X}, \Hof{y} &\ge 0 \\
      \MIof{y ; X} &\le \Hof{y},
    \intertext{and if \(X\) is also a discrete random variable (or we inject zero-entropy noise), we gain:}
      \MIof{y ; X} &\le \simpleE{\pof{x \given y}}{\Hof{x}}.
  \end{align}
\end{proposition}
\begin{proof}
  Again, the first two statements follow from:
  \begin{align}
    \Hof{y} - \Hof{y \given X} &= \MIof{y ; X} \notag \\
    &= \simpleE{\pof{x \given y}}{\MIof{y ; x}} \notag \\
    &= \E{\pof{x \given y}}{\Hof{x} - \Hof{x \given y}} \label{eq:proof_inequality_Hy}\\
    &= \Kale{\pof{X \given y}}{\pof{X}} \notag \\
    &\ge 0.
  \end{align}
  The third statement follows from \cref{eq:proof_inequality_Hy} above as \(0 \le \E{\pof{x \given y}}{\Hof{x} - \Hof{x \given y}} = \simpleE{\pof{x \given y}}{\Hof{x} - \Hof{X \given y}}\).
  The fourth statement follows from \(\pof{y \given x} \le 1\) when \(Y\) is a discrete random variable, and thus \(\Hof{y \given X} \ge 0\) due to the monotony of the expectation. When we inject zero-entropy noise, we similarly have \(\pof{y \given x} \le 1\) for almost all $y$ as otherwise $\Hof{X} \le 0$ in contradiction to \cref{eq:zero-entropy-noise}.
  The fifth statement follows from the fourth statement and \(\MIof{y ; X} = \Hof{y} - \Hof{y \given X} \le \Hof{y}\).
  Finally, if \(X\) is a discrete random variable as well, we also have \(\Hof{X \given y} \ge 0\), and thus
  \begin{align*}
    \MIof{y ; X} = \E{\pof{x \given y}}{\Hof{x} - \Hof{X \given y}} \le \simpleE{\pof{x \given y}}{\Hof{x}}.
  \end{align*}
  Similarly, when we inject zero-entropy noise into $X$, we also \(\Hof{X \given y} \ge 0\) following \cref{eq:zero-entropy-noise} as the noise is assumed to be independent.
\end{proof}
Note that there are no such general bound for \(\MIof{X ; y}\), \(\Hof{X \given y}\) and \(\Hof{y \given X}\).

\begin{corollary}
  We have $\MIof{y ; X} = 0$ exactly when $\pof{x \given y} = \pof{x}$ for all $x$ for given $y$. 
\end{corollary}
\begin{proof}
  This follows from $0 = \MIof{y ; X} = \Kale{\pof{x \given y}}{\pof{x}}$ exactly when $\pof{x \given y} = \pof{x}$.
\end{proof}

In particular, there is a misleading intuition that the information gain \(\MIof{X ; y} = \Hof{X} - \Hof{X \given y}\) ought to be non-negative for any $y$. This is not true. This intuition may exist because in many cases when we look at posterior distributions, we only model the mean and assume a fixed variance of these distributions. The uncertainty around the mean does indeed reduce with additional observations; however, the uncertainty around the variance might not.\yarin{maybe give some visualisations for intuition?} The reader is invited to experiment with a normal distribution with known mean and compute the information gain on the variance depending on new observations. 

In a sense, the information-theoretic surprise is much better behaved than the information gain because we can bound it in various ways, which does not seem possible for the information gain. The information gain is a more useful quantity though for active learning and Bayesian optimal experimental design. As such it is useful to have a unified notation that includes both quantities.

\section{Example Application: Stirling's Approximation for Binomial Coefficients}
\label{sec:stirling_binomial_coefficient}
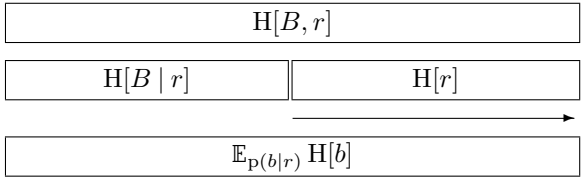
\begin{figure}
  \begin{center}
  \setlength{\unitlength}{1in}
  \begin{picture}(3,0.9)(0,0.0)
  \put(0,0.7){\framebox(3,0.20){\(\Hof{B, r}\)}}
  \put(1.5,0.4){\framebox(1.5,0.20){\(\Hof{r}\)}}
  \put(0,0.4){\framebox(1.475,0.20){\(\Hof{B \given r}\)}}
  \put(1.5,0.3){\vector(1,0){1.475}}
  \put(0,0.0){\framebox(3,0.20){\(\simpleE{\pof{b \given r}}{\Hof{b}}\)}}
  \end{picture}
  
  \end{center}
  {%
  \caption{\emph{The relationship between between the information quantities used in \S\ref{sec:stirling_binomial_coefficient}.} \(B\) is the joint of the binomial random variables, \(R\) is the number of successes in \(B\) with observed outcome \(r\). The arrow below \(\Hof{r}\) symbolizes that we minimize \(\Hof{r}\) by optimizing the success probability \(\rho\) to close the gap between \(\simpleE{\pof{b \given r}}{\Hof{b}}\) and \(\Hof{B \given r}\).
  }
  \label{fig:stirling_approx_for_binomial_coeffs}
  }%
\end{figure}

In \citet{mackay2003information} on page 2, the following simple approximation for a binomial coefficient is introduced:
\begin{equation}
    \log \binom{N}{r} \simeq(N-r) \log \frac{N}{N-r}+r \log \frac{N}{r}. \label{eq:stirling_approx}
\end{equation}
We will derive this result using the proposed extension to observed outcomes as it allows for an intuitive deduction. Moreover, we will see that this allows us to use other tools from probability theory to estimate the approximation error.

\textbf{Setup.} %
Let $B_1, \ldots, B_N$ be $N$ Bernoulli random variables with success probability $p$, and let $B$ be the joint of these random variables.

Further, let $R$ be the random variable that counts the number of successes in $B$. $R$ follows a Binomial distribution with success probability $\rho$ and $N$ trials.

\textbf{Main Idea.} %
For a given outcome $r$ of $R$, we have:
\begin{equation}
    \Hof{B, r} = \Hof{B \given r} + \Hof{r} \ge \Hof{B \given r},
\end{equation}
as $\Hof{\cdot}$ is non-negative for discrete random variables.
We will examine this inequality to obtain the approximation in \cref{eq:stirling_approx}.

Note that $\Hof{B \given r}$ is the additional number of bits needed to encode $B$ when the number of successes is already known. Similarly, $\Hof{B, r}$ is the number of bits needed to encode both $B$ and $R$ under the circumstance that $R=r$. 

\textbf{Determining $\Hof{B, r}$.} %
$R$ is fully determined by $B$, and thus we have $\Hof{B, R}=\Hof{B}$ and hence\footnote{This also follows immediately from $\Hof{R \given B} = 0 \implies \forall r: \Hof{r \given B} = 0$.}:
\begin{equation}
    \Hof{B, r} = \simpleE{\pof{b \given r}}{\Hof{b}}.
\end{equation}
$\simpleE{\pof{b \given r}}{\Hof{b}}$ is the expected number of bits needed to transmit the outcome $b$ of $B$ when $r$ is given.
When we encode $B$, we do not know $r$ upfront, so we need to transmit $N$ Bernoulli outcomes.
Hence, we need to transmit $r$ successes and $N-r$ failures. Given the success probability $\rho$, the optimal message length for this is:
\begin{align}
    &\simpleE{\pof{b \given r}}{\Hof{b}} = r \, \ICof{\rho} + (N-r)\, \ICof{1-\rho} \\
    & \quad = - r \log \rho - (N - r) \log (1-\rho).
\end{align}
All this is visualized in \Cref{fig:stirling_approx_for_binomial_coeffs}.

\textbf{Alternative Argument.}
We can also look at the terms $\Hof{B \given r} + \Hof{r}$ separately. We have
\begin{equation}
    \Hof{r} = -\log \pof{r} = - \log \left ( \binom{N}{r} \, \rho^r \,(1-\rho)^{N-r}\right ),
\end{equation}
and
\begin{equation}
    \Hof{B \given r} = -\simpleE{\pof{b \given r}} \log \pof{b \given r} = 
    \log \binom{N}{r}.
\end{equation}
The former follows from $R$ being binomially distributed.
For the latter, we observe that we need to encode $B$ while knowing $r$ already. Given $r$, $\pof{b \given r} = \text{const}$ for all valid $b$. There are $\binom{N}{r}$ possible $b$ for fixed $r$.
Hence, we can simply create a table with all possible configurations with $r$ successes. There are $\binom{N}{r}$ many. We then encode the index into this table. 

Each configuration with $r$ successes has an equal probability of happening, so we have a uniform discrete distribution with entropy $\log \binom{N}{r}$ and obtain the same result.

\textbf{Determining $\rho$.}
We already have
\begin{align}
\Hof{B \given r} + \Hof{r} &=  -r \log \rho - (N - r) \log (1-\rho) \notag \\
&\ge
\log \binom{N}{r} = \Hof{B \given r}. \label{eq:main_result}
\end{align}
How do we make this inequality as tight as possible?

We need to minimize the gap $\Hof{r}$ which creates the inequality in the first place, and $\Hof{r}=-\log \pof{r}$ is minimized exactly when $\pof{r}$ becomes maximal.

Hence, we choose the success probability $\rho$ to do so: the maximum likelihood solution $\argmax_p \pof{r \given \rho}$ is $\rho = \frac{r}{N}$.
The Binomial distribution of $R$ then has its mode, mean, and median at $r$.

Altogether, after substituting $\rho = \frac{r}{N}$ and rearranging, we see that the wanted approximation is actually an inequality:
\begin{align}
    \log \binom{N}{r} &\le -r \log \rho - (N - r) \log (1-\rho) \\
& = r \log \frac{N}{r} + (N - r) \log \frac{N}{N-r}. 
\end{align}

\textbf{Approximation Error $\Hof{r}$.} %
The approximation error is just $\Hof{r}$ as we can read off from \cref{eq:main_result}. We can easily upper-bound it with $\Hof{r} \le \log N$: 
First, $\Hof{R} \le \log N$ as the uniform distribution with entropy $\log N$ is the maximum entropy distribution in this case (discrete random variable with finite support).
Second, $\Hof{R}$ is the expectation over different $\Hof{R=r'}$. We have chosen $\rho=\tfrac{r}{N}$ such that $r$ is the mean of binomial distribution and has maximal probability mass. This means it has minimal information content. Hence $\Hof{r} \le \log N$ by contraposition as otherwise $\log N < \Hof{r} \le \Hof{R}$.

\andreas{can I write that out and see how well it approximates the binomial coefficient itself?}

\section{Example Application: ELBO of a Variational Auto-Encoder}

The specific evidence lower bound inequality (ELBO) developed in \citet{kingma2014autoencoding} is a useful tool. The derivation in the paper has been described as hard to follow, however. We can elegantly derive the\yarin{using this derivation we gain new insight about [BLAH] which we describe at the end of the chapter

[maybe IB?]} relevant inequality at a high level using our practical notation and Bayes' theorem.

\textbf{Variational Auto-Encoder.} We have a probabilistic model \(\pof{x,z} \coloneqq \pcof{\theta}{z} \pcof{\theta}{x \given z}\) of observed \(X\) given some hidden latent variable \(Z\) with parameters \(\theta\). Usually \(\pcof{\theta}{z}\) is fixed as \(\pof{z}\) and follows a simple distribution: a unit Gaussian, for example. We desire to learn a variational approximation \(\qcof{\phi}{z \given x}\) with parameters \(\phi\) of \(\pof{z \given x}\),\yarin{$\{z_i\}_N | \{x_i\}_N$
} where the latter might be intractable. \(\pof{x}\) is only available implicitly through the available training data, which means that we can sample from it but not compute the density directly. This is where the ELBO comes in.

\textbf{ELBO.} Following\yarin{also Shakir's DGM paper from the same year} \citet{kingma2014autoencoding}, we minimize a forward KL divergence as variational objective: when \(\Kale{\qcof{\phi}{Z \given X}}{\pof{Z \given X}} = 0\), we also\yarin{almost everywhere
(from memory, this is equality in dist)
can also use
$=^D$} have \(\qcof{\phi}{z \given x} = \pof{z \given x}\). In this case, we have found a consistent variational approximation.
But in general this does not hold.
Depending on the quality of the approximation, we can draw approximate samples of \(\pof{x}\) by first sampling \(z \sim \pcof{\theta}{z}\) and then sampling \(x \sim \pcof{\theta}{x \given z}\).
\begin{proposition}
Minimizing\yarin{explain to reader that you're gonna give a derivation using your notation

prob will be good to give the standard derivation as well to contrast to yours} the forward KL divergence \(\Kale{\qcof{\phi}{Z \given X}}{\pof{Z \given X}} \ge 0\) is equivalent to maximizing (the left-hand side in) the evidence lower bound \(\E{\pof{x} \qcof{\phi}{z \given x}}{\log{\pcof{\theta}{x \given z}}} - \Kale{\qcof{\phi}{Z \given X}}{\pcof{\theta}{Z}} \le \simpleE{\pof{x}}\log{\pof{x}}\).
\end{proposition}
\begin{proof}
  We begin with an information-theoretic deduction which is straightforward using Bayes' theorem and the rules in \Cref{prop:xe_kale_rules}:
  \begin{align*}
    0 &\le \Kale{\qcof{\phi}{Z \given X}}{\pof{Z \given X}} \\
    &= \CrossEntropy{\qcof{\phi}{Z \given X}}{\underbrace{\pof{Z \given X}}_{
      = \frac{\pcof{\theta}{X \given Z} \pcof{\theta}{Z}}{\pof{X}}
    }} - \xHof{\qcof{\phi}{Z \given X}} \\
    &= \CrossEntropy{\qcof{\phi}{Z \given X}}{\pcof{\theta}{X \given Z}} + \CrossEntropy{\qcof{\phi}{Z \given X}}{\pcof{\theta}{Z}} \\
    & \quad - \xHof{\pof{X}} - \xHof{\qcof{\phi}{Z \given X}} \\
    &= \CrossEntropy{\qcof{\phi}{Z \given X}}{\pcof{\theta}{X \given Z}}  \\
    & \quad + \Kale{\qcof{\phi}{Z \given X}}{\pcof{\theta}{Z}} - \xHof{\pof{X}}
  \end{align*}
  Finally, we can rearrange and expand the definitions to obtain the ELBO:
  \begin{align*}
    \xHof{\pof{X}} &\le \E{\pof{x} \qcof{\phi}{z \given x}}{\ICof{\pcof{\theta}{x \given z}}} \\
    &\quad + \Kale{\qcof{\phi}{Z \given X}}{\pcof{\theta}{Z}} \\
    \Leftrightarrow \simpleE{\pof{x}}\log{\pof{x}} & \ge \simpleE{\pof{x}}\E{\qcof{\phi}{z \given x}}{\log{\pcof{\theta}{x \given z}}} \\
    & \quad - \Kale{\qcof{\phi}{Z \given X}}{\pcof{\theta}{Z}}.
  \end{align*}
\end{proof}
From an information-theoretic perspective, the ELBO is actually an \emph{upper-bound} on the\yarin{and therefore...
[why do we care about this insight? give some followups that we can deduce from this view / derivation]} entropy of the inputs: 
\begin{align*}
  \Hof{X} \le &\CrossEntropy{\qcof{\phi}{Z \given X}}{\pcof{\theta}{X \given Z}} \\
  & + \Kale{\qcof{\phi}{Z \given X}}{\pcof{\theta}{Z}}.
\end{align*}
Note that in comparison to \citet{kingma2014autoencoding}, we take an expectation over \(x\)\yarin{in line \#

also this is related to Tom's derivation and empirical estimate view of the ELBO} right away. The non-expected version would be 
\begin{align}
  \Hof{x} \le &\CrossEntropy{\qcof{\phi}{Z \given x}}{\pcof{\theta}{x \given Z}} \\
    & + \Kale{\qcof{\phi}{Z \given x}}{\pcof{\theta}{Z}}.
\end{align}

\begin{nextversion}
\textbf{Information-Theoretic Deduction of the ELBO.} Following appendix F.5 in \citet{kirsch2020unpacking} we can also easily deduce the ELBO in a different more information-theoretic way. 

For this, we assume the probabilistic model \(\pof{x, z}\) with \(\pof{x}\) determined by the given data distribution and \(\pof{z}\) fixed to a unit Gaussian. We want to approximate \(\pof{x \given z}\) using a variational distribution \(\pcof{\theta}{x \given z}\) and \(\pof{z \given x}\) using \(\pcof{\phi}{z \given x}\).

TK TK NOT CLEAR TK TK 
\end{nextversion}

\section{Example Application: Variational Inference, Active Learning, and Core-Set Methods}

\yarin{split into separate chapters}
\yarin{ok not to move to bg
but be sure to be clear where the BG ends and your contribution begins

so have a BG subsection for each chapter}
We start by briefly revisiting Bayesian deep learning, variational inference, and the evidence-lower-bound inequality, before introducing active learning and defining the Core-Set by Disagreement acquisition function.

\textbf{Probabilistic Model.}
The model parameters are treated as a random variable $\W$ with prior distribution $\pof{\w}$. We denote the training set $\Dtrain=\{ ( \xtrain_i,\ytrain_i) \}_{1, \ldots, i \in |\Dtrain|}$, where $\xtrainsetfull$ are the input samples and $\ytrainsetfull$ the labels or targets.
The \emph{probabilistic model} is as follows:
\begin{equation}
    \pof{y, x, \w} = \pof{y \given x, \w} \, \pof{\w} \, \pof{x},
\end{equation}
where $x$, $y$, and $\omega$ are outcomes for the random variables $X$, $Y$, and $\Omega$ denoting the input, label, and model parameters, respectively. %

To include multiple labels and inputs, we expand the model to joints of random variables $\xsetfull$ and $\ysetfull$ obtaining
\begin{align}
    \pof{\yset,\xset,\w}=\prod_{i \in I} \pof{y_i \given x_i, \w} \pof{x_i} \pof{\w}.
\end{align}
We are only interested in discriminative models and thus do not explicitly model $\pof{x}$.

The posterior parameter distribution $\pof{\w \given \Dtrain}$ is determined via Bayesian inference. We obtain $\pof{\w \given \Dtrain}$ using Bayes' theorem:
\begin{equation}
    \pof{\w \given \Dtrain} \propto \pof{\ytrainset \given \xtrainset, \w} \pof{\w}.
\end{equation}
which allows for predictions by marginalizing over $\W$:
\begin{equation}
    \pof{y \given x, \Dtrain} = 
    \simpleE{\w \sim \pof{\w \given \Dtrain}}{\pof{y \given x, \w}}.
\end{equation}

\textbf{Variational Inference \& ELBO.} %
Exact Bayesian inference is intractable for complex models, and we use variational inference for approximate inference using a variational distribution $\qof{\w}$. We can determine $\qof{\w}$ by minimizing the following KL divergence:
\begin{multline}
    \Kale{\qof{\w}}{\pof{\w \given \Dtrain}} = \\ 
    = \underbrace{\CrossEntropy{\qof{\w}}{\pof{\ytrainset \given \xtrainset, \omega}}}_\text{likelihood} \\
    + \underbrace{\Kale{\qof{\w}}{\pof{\w}}}_\text{prior regularization} + \underbrace{\log \pof{\Dtrain}}_\text{model evidence} \ge 0,
\end{multline}
where we used Bayes' theorem. We proof this using the notation from this paper as an application:
\begin{proposition}
  Minimizing the forward KL divergence \(\Kale{\qof{\w}}{\pof{\w \given \Dtrain}}\) is equivalent to maximizing the evidence lower-bound (ELBO) \(\sum_i \simpleE{\qof{\w}}{\log \pof{\ytrain_i \given \xtrain_i, \w}} - \Kale{\qof{\w}}{\pof{\w}}\).
\end{proposition}
\begin{proof}
  We start with the information-theoretic deduction which is straightforward using Bayes' theorem and the rules\yarin{again, be clear that you give a derivation using your notation

  and explain to the reader why they should care about this derivation - eg it allows you to give new insights, or you will use this derivation to connect this to something new (which I think is missing in the writeup atm)} in \Cref{prop:xe_kale_rules}:
  \begin{align*}
    & 0 \le \Kale{\qof{\w}}{\pof{\w \given \Dtrain}} = \\
    & \quad = \CrossEntropy{\qof{\w}}{\pof{\w \given \Dtrain}} - \xHof{\qof{w}} \\
    & \quad = \CrossEntropy{\qof{\w}}{\frac{\pof{\ytrainset \given \xtrainset, \w} \pof{\w}}{\pof{\Dtrain}}} - \xHof{\qof{w}} \\
    & \quad = \CrossEntropy{\qof{\w}}{\pof{\ytrainset \given \xtrainset, \w}} \\
    & \quad \quad + \CrossEntropy{\qof{\w}}{\pof{\w}} - \xHof{\qof{w}} - \xHof{\pof{\Dtrain}} \\
    & \quad = \CrossEntropy{\qof{\w}}{\pof{\ytrainset \given \xtrainset, \w}} \\
    & \quad \quad + \Kale{\qof{\w}}{\pof{\w}} + \ICof{\pof{\Dtrain}}.
  \end{align*}
  It thus follows:
  \begin{align*}
    \ICof{\pof{\Dtrain}} &\le \CrossEntropy{\qof{\w}}{\pof{\ytrainset \given \xtrainset, \w}} \\
    & \quad  + \Kale{\qof{\w}}{\pof{\w}}.
  \end{align*}
  Expanding the definitions, we obtain
  \begin{align*}
    - \log \pof{\Dtrain} \le &\E{\qof{\w}}{-\log \prod_i \pof{\ytrain_i \given \xtrain_i, \w}} \\
      &+ \Kale{\qof{\w}}{\pof{\w}},
  \end{align*}
  and after some rearranging, the ELBO surfaces:
  \begin{align*}
    \log \pof{\Dtrain} \ge &\sum_i \E{\qof{\w}}{\log \pof{\ytrain_i \given \xtrain_i, \w}} \\
    &- \Kale{\qof{\w}}{\pof{\w}},
  \end{align*}
  with equality exactly\yarin{this section feels incomplete; also $=^D$} when \(\qof{\w} = \pof{\w \given \Dtrain}\).
\end{proof}
For Bayesian deep learning models,\yarin{this section feels incomplete} we can use the local reparameterization trick or Monte-Carlo dropout for $\qof{\w}$ \citep{kingma2015variational, gal2016dropout}.

\textbf{Active Learning.}
In active learning, we have access to an unlabelled pool set $\Dpool=\xpoolsetfull$. We iteratively acquire batches of samples $\xbatchset$ from the pool set into the training set by acquiring labels for them through an oracle and retrain our model.
We repeat these steps until the model satisfies our performance requirements.

To determine which samples to select for acquisition, 
we score candidate acquisition batches $\xbatchset$ with the \emph{acquisition batch size} b using an acquisition function $a(\xbatchset, \pof{\W \given \Dtrain})$ and pick the highest scoring one:
\begin{equation}
    \underset{\displaystyle\{\xbatch_i\}_{i \in\{1,\ldots,b\}} \subseteq \Dpool}{\argmax} a(\xbatchset, \pof{\W \given \Dtrain})
\end{equation}
BALD was originally introduced as a one-sample acquisition function of the expected information gain between the prediction $\Ybatch$ for a candidate input $\xbatch$ and the model parameters $\W$:
\begin{math}
  \MIof{\W; \Ybatch \given \xbatch, \Dtrain}.
\end{math}
In BatchBALD \citep{kirsch2019batchbald}, this one-sample case was canonically extended to the batch acquisition case using the expected information gain between the \emph{joint} of the predictions $\Ybatchset$ for the batch candidates $\xbatchset$ and the model parameters $\W$:
\begin{multline}
    a_\text{BALD}(\xbatchset, \pof{\W \given \Dtrain}) := \\
    = \MIof{\W ; \Ybatchset \given \xbatchset, \Dtrain}
\end{multline}

\textbf{Notation.}
Instead of $\Yevalset$, $\xevalset$, we will write $\boldsymbol{\Yeval}$, $\boldsymbol{\xeval}$ and so on to to cut down on notation. Like above, all terms can be canonically extended to sets by substituting the joint. 
Lower-case variables like $\yeval$ are outcomes of random variables while upper-case variables like $\Yeval$ are random variables. The datasets $\Dpool, \Dtrain$ are sets of outcomes.

{
\renewcommand{\xevalset}{\boldsymbol{\xeval}}
\renewcommand{\xtestset}{\boldsymbol{\xtest}}
\renewcommand{\xtrainset}{\boldsymbol{\xtrain}}
\renewcommand{\xbatchset}{\boldsymbol{\xbatch}}
\renewcommand{\xpoolset}{\boldsymbol{\xpool}}

\renewcommand{\Xevalset}{\boldsymbol{\Xeval}}
\renewcommand{\Xtestset}{\boldsymbol{\Xtest}}
\renewcommand{\Xtrainset}{\boldsymbol{\Xtrain}}
\renewcommand{\Xbatchset}{\boldsymbol{\Xbatch}}
\renewcommand{\Xpoolset}{\boldsymbol{\Xpool}}

\renewcommand{\yevalset}{\boldsymbol{\yeval}}
\renewcommand{\ytestset}{\boldsymbol{\ytest}}
\renewcommand{\ytrainset}{\boldsymbol{\ytrain}}
\renewcommand{\ybatchset}{\boldsymbol{\ybatch}}
\renewcommand{\ypoolset}{\boldsymbol{\ypool}}

\renewcommand{\Ytestset}{\boldsymbol{\Ytest}}
\renewcommand{\Ytrainset}{\boldsymbol{\Ytrain}}
\renewcommand{\Yevalset}{\boldsymbol{\Yeval}}
\renewcommand{\Ybatchset}{\boldsymbol{\Ybatch}}
\renewcommand{\Ypoolset}{\boldsymbol{\Ypool}}

\subsection{BALD $\to$ Core-Set by Disagreement}
\yarin{for the thesis, you can also add the paper with Tom as another chapter!} 
We examine BALD through the lens of our new notation and\yarin{use our notation to draw links to a different problem - core set selection. This allows us to develop ..} develop CSD as information gain. First, we note that BALD does not optimize the loss of the test distribution to become minimal. It does not try to pick labels which minimize the generalization loss.

BALD maximizes the expected information gain: $\MIof{\W; \Ybatchset \given \xbatchset, \Dtrain} = \Hof{\W} - \Hof{\W \given \Ybatchset, \xbatchset, \Dtrain}$. We assume that our Bayesian model contains the true generating model parameters and by selecting samples that minimize the uncertainty $\Hof{\W \given \Ybatchset, \xbatchset, \Dtrain}$, the model parameters will converge towards these true parameters as $\Hof{\W \given \Dtrain} \to 0$.

\textbf{BALD as an Approximation.}
BALD as the expected information gain is the expectation of the information gain over the current model's predictions for $x$:
\begin{equation}
  \MIof{\W; Y \given x, \Dtrain} = \simpleE{\pof{y \given x, \Dtrain}} \MIof{\W; y \given x, \Dtrain}.
\end{equation}
Using the definition,\yarin{hence
$E_p(y|..) pmi = MI$} we have:
\begin{align}
  \MIof{\W; y \given x, \Dtrain} = \Hof{\W \given \Dtrain} - \Hof{\W \given y, x, \Dtrain}. \label{eq:ig}
\end{align}
That is, we can view BALD as weighting the information gains $\MIof{\W; y \given x, \Dtrain}$ for different $y$ by the current model's belief that $y$ is correct.\andreas{cite Active Testing}\yarin{We can use this connection to draw multiple useful insights.
[put the next few sentences into their own para / subsections and give more info]}
If we had access to the labels or a better surrogate distribution for the labels, we could improve on this. This could in particular help with the cold starting problem\yarin{we can give some early exps showing this} in active learning when one starts training with no initial training set and the model predictions are not trustworthy at all.\andreas{add ablation}
When we have access to the labels, we can directly use the \emph{information gain} $\MIof{\W; y^\text{true} \given x, \Dtrain}$ and select the samples using a \emph{Core-Set by Disagreement} acquisition function:
\begin{multline}
  a_\text{CSD}(\ybatchset, \xbatchset, \pof{\W \given \Dtrain}) := \\
    = \MIof{\W ; \ybatchset \given \xbatchset, \Dtrain}
\end{multline}
\textbf{Evaluating the Information Gain.} We show how to compute the information for the special case of an MC dropout model with dropout rate $\tfrac{1}{2}$. Computing the information gain for other models is not trivial as it usually requires an explicit density model. Most approximate Bayesian neural networks, such as Monte-Carlo dropout models, only provide implicit models, which we can sample from but which do not provide a way to approximate the posterior density. 
Moreover, to compute $\Hof{\W \given y, x, \Dtrain}$ naively we would have to perform a Bayesian inference step. We can rewrite, however:
\begin{align}
  &\MIof{\W; y \given x, \Dtrain} \notag \\
  &=\Hof{\W \given \Dtrain} - \Hof{\W, y \given x,\Dtrain} + \Hof{y \given x, \Dtrain} \\
  &=\Hof{\W \given \Dtrain} + \Hof{y \given x, \Dtrain} \notag \\
  &\quad - \simpleE{\pof{\w \given y, x, \Dtrain}} \Hof{\w, y \given x,\Dtrain}. \label{eq:expanded_bald}
\end{align}
We can expand $\simpleE{\pof{\w \given y, x, \Dtrain}} \Hof{\w, y \given x,\Dtrain}$ to:
\begin{align}
  &\simpleE{\pof{\w \given y, x, \Dtrain}} \Hof{\w, y \given x,\Dtrain} \notag \\
  &= \E{\pof{\w \given y, x, \Dtrain}}{{\Hof{\w \given \Dtrain}} + {\Hof{y \given x, \w, \Dtrain}}} \\
  &= \E{\pof{\w \given y, x, \Dtrain}}{{\Hof{\w \given \Dtrain}}} + {\Hof{y \given x, \Omega, \Dtrain}}.
\end{align}
Plugging everything into the \eqref{eq:expanded_bald} and rearranging, we obtain:
\begin{align}
  &\MIof{\W; y \given x, \Dtrain} = \notag \\
  & = \underbrace{\simpleE{\pof{\w \given x, \Dtrain}} \Hof{\w \given \Dtrain} - \simpleE{\pof{\w \given y, x, \Dtrain}} \Hof{\w \given \Dtrain}}_{\Circled{3}} \notag \\
  & \quad + \underbrace{\Hof{y \given x, \Dtrain} - \Hof{y \given x, \W, \Dtrain}}_{\displaystyle =\MIof{y ; \W \given x, \Dtrain}}.
\end{align}
To compute $\Hof{y \given x, \W, \Dtrain}$, we use importance sampling:
\begin{align}
  &\Hof{y \given x, \W, \Dtrain} = \simpleE{\pof{\w \given y, x, \Dtrain}} \Hof{y \given x, \w} = \notag \\
  & \quad = \simpleE{\pof{\w \given \Dtrain}} \frac{\pof{\w \given y, x, \Dtrain}}{\pof{w \given \Dtrain}} \Hof{y \given x, \w} \notag \\
  & \quad =\simpleE{\pof{\w \given \Dtrain}} \frac{\pof{y \given x, \w}}{\pof{y \given x, \Dtrain}} \Hof{y \given x, \w}
\end{align}
Finally, if we use Monte-Carlo dropout with dropout rate $\tfrac{1}{2}$ to obtain a variational model distribution $\qof{\w}$, we have $\qof{\w} = \text{const}$, and we can approximate \Circled{3} as:
\begin{align}
  & \simpleE{\pof{\w \given \Dtrain}} \Hof{\w \given \Dtrain} - \simpleE{\pof{\w \given y, x, \Dtrain}} \Hof{\w \given \Dtrain} = \notag \\
  & \quad \approx \simpleE{\pof{\w \given \Dtrain}} \ICof{\qof{\w}} - \simpleE{\pof{\w \given y, x, \Dtrain}} \ICof{\qof{\w}} \notag \\
  & \quad = \ICof{\qof{\w}} - \ICof{\qof{\w}} = 0.
\end{align}
In this special case, we indeed have $\MIof{\W; y \given x, \Dtrain} = \MIof{y; \W \given x, \Dtrain}$. We can use the surprise to approximate the information gain. 

As an example, \Cref{fig:coresetbald_mnist} shows that CSD strongly outperforms BALD on MNIST in this setup. This approximation is brittle, however. We study this in experiments in \S\ref{sec:csd_experiments} in the appendix.

\begin{figure}[t]
  \centering
  \includegraphics[width=\linewidth]{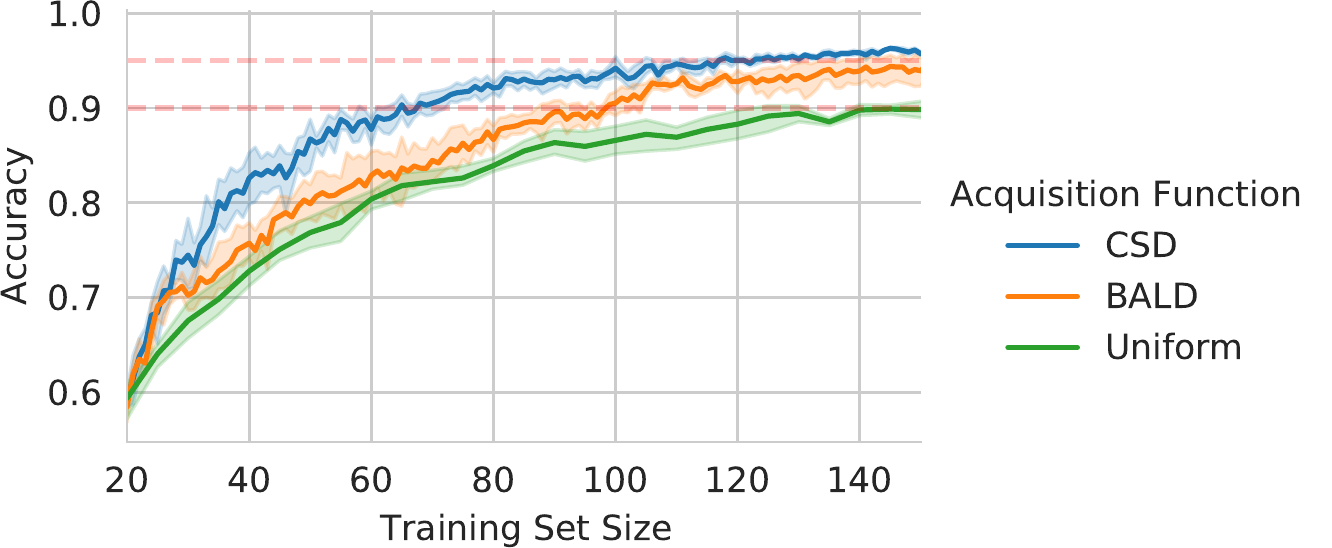}
  \caption{
    \emph{CSD vs BALD vs uniform acquisition on MNIST after ambiguous and mislabeled training samples have been removed from the training set.}\yarin{I would expect to see comparison to other coreset approaches - why is BALD an appropriate baseline here?} CSD requires only 58 samples to reach 90\% accuracy compared to 91 samples for BALD. 5 trials each. Dashed horizontal lines at 90\% and 95\% accuracy.\andreas{move this to the CSD section and the rest into the appendix.}
  }
  \label{fig:coresetbald_mnist}
\end{figure}

\section{Conclusion}

We have introduced a unified notation for information-theoretic quantities for both random variables and outcomes. We have also unified information gain and surprise by defining the mutual information appropriately. Finally, we have examined applications of our notation which show potential avenues for future research. This shows that our notation allows for new perspectives on well-known problems that simplify thinking about them---a strong signal that it is a useful abstraction.

\section*{Acknowledgements}

The authors would like to thank Joost van Amersfoort,\yarin{I think Lewis could give some useful feedback on your derivations as well!} Tim Rudner, Ravid Shwartz-Ziv, Clare Lyle, as well as the members of OATML in general for their feedback at various stages of the project. AK is supported by the UK EPSRC CDT in Autonomous Intelligent Machines and Systems (grant reference EP/L015897/1).

\bibliographystyle{plainnat}
\bibliography{references}

\begin{thebibliography}{24}
\providecommand{\natexlab}[1]{#1}
\providecommand{\url}[1]{\texttt{#1}}
\expandafter\ifx\csname urlstyle\endcsname\relax
  \providecommand{\doi}[1]{doi: #1}\else
  \providecommand{\doi}{doi: \begingroup \urlstyle{rm}\Url}\fi

\bibitem[Bellemare et~al.(2016)Bellemare, Srinivasan, Ostrovski, Schaul,
  Saxton, and Munos]{bellemare2016unifying}
Marc~G. Bellemare, Sriram Srinivasan, Georg Ostrovski, Tom Schaul, David
  Saxton, and Remi Munos.
\newblock Unifying count-based exploration and intrinsic motivation, 2016.

\bibitem[Butts(2003)]{butts2003much}
Daniel~A Butts.
\newblock How much information is associated with a particular stimulus?
\newblock \emph{Network: Computation in Neural Systems}, 14\penalty0
  (2):\penalty0 177--187, 2003.

\bibitem[Church and Hanks(1990)]{church1990word}
Kenneth~Ward Church and Patrick Hanks.
\newblock Word association norms, mutual information, and lexicography.
\newblock \emph{Computational Linguistics}, 16\penalty0 (1):\penalty0 22--29,
  1990.
\newblock URL \url{https://aclanthology.org/J90-1003}.

\bibitem[Cover and Thomas(2006)]{CTJ}
Thomas~M. Cover and Joy~A. Thomas.
\newblock \emph{Elements of Information Theory}.
\newblock Wiley-Interscience, USA, 2006.
\newblock ISBN 0471241954.

\bibitem[DeWeese and Meister(1999)]{deweese1999measure}
Michael~R DeWeese and Markus Meister.
\newblock How to measure the information gained from one symbol.
\newblock \emph{Network: Computation in Neural Systems}, 10\penalty0
  (4):\penalty0 325--340, 1999.

\bibitem[Fano(1962)]{fano1962transmission}
Robert~M. Fano.
\newblock Transmission of information, 1962.

\bibitem[Foster et~al.(2019)Foster, Jankowiak, Bingham, Horsfall, Teh,
  Rainforth, and Goodman]{foster2019variational}
Adam Foster, Martin Jankowiak, Eli Bingham, Paul Horsfall, Yee~Whye Teh, Tom
  Rainforth, and Noah Goodman.
\newblock Variational bayesian optimal experimental design.
\newblock \emph{arXiv preprint arXiv:1903.05480}, 2019.

\bibitem[Gal and Ghahramani(2016)]{gal2016dropout}
Yarin Gal and Zoubin Ghahramani.
\newblock Dropout as a bayesian approximation: Representing model uncertainty
  in deep learning.
\newblock In \emph{international conference on machine learning}, pages
  1050--1059. PMLR, 2016.

\bibitem[Gal et~al.(2017)Gal, Islam, and Ghahramani]{gal2017deep}
Yarin Gal, Riashat Islam, and Zoubin Ghahramani.
\newblock Deep bayesian active learning with image data.
\newblock In \emph{International Conference on Machine Learning}, pages
  1183--1192. PMLR, 2017.

\bibitem[Houlsby et~al.(2011)Houlsby, Husz{\'a}r, Ghahramani, and
  Lengyel]{houlsby2011bayesian}
Neil Houlsby, Ferenc Husz{\'a}r, Zoubin Ghahramani, and M{\'a}t{\'e} Lengyel.
\newblock Bayesian active learning for classification and preference learning.
\newblock \emph{arXiv preprint arXiv:1112.5745}, 2011.

\bibitem[J{\'o}nsson et~al.(2020)J{\'o}nsson, Cherubini, and
  Eleftheriou]{jonsson2020convergence}
Hlynur J{\'o}nsson, Giovanni Cherubini, and Evangelos Eleftheriou.
\newblock Convergence behavior of dnns with mutual-information-based
  regularization.
\newblock \emph{Entropy}, 22\penalty0 (7):\penalty0 727, 2020.

\bibitem[Kingma and Welling(2014)]{kingma2014autoencoding}
Diederik~P Kingma and Max Welling.
\newblock Auto-encoding variational bayes, 2014.

\bibitem[Kingma et~al.(2015)Kingma, Salimans, and
  Welling]{kingma2015variational}
Diederik~P. Kingma, Tim Salimans, and Max Welling.
\newblock Variational dropout and the local reparameterization trick, 2015.

\bibitem[Kirsch et~al.(2019)Kirsch, van Amersfoort, and
  Gal]{kirsch2019batchbald}
Andreas Kirsch, Joost van Amersfoort, and Yarin Gal.
\newblock Batchbald: Efficient and diverse batch acquisition for deep bayesian
  active learning.
\newblock In \emph{Advances in Neural Information Processing Systems}, pages
  7024--7035, 2019.

\bibitem[Kirsch et~al.(2020)Kirsch, Lyle, and Gal]{kirsch2020unpacking}
Andreas Kirsch, Clare Lyle, and Yarin Gal.
\newblock Unpacking information bottlenecks: Unifying information-theoretic
  objectives in deep learning.
\newblock \emph{arXiv preprint arXiv:2003.12537}, 2020.

\bibitem[LeCun et~al.(1998)LeCun, Bottou, Bengio, and
  Haffner]{lecun1998gradient}
Yann LeCun, L{\'e}on Bottou, Yoshua Bengio, and Patrick Haffner.
\newblock Gradient-based learning applied to document recognition.
\newblock \emph{Proceedings of the IEEE}, 86\penalty0 (11):\penalty0
  2278--2324, 1998.

\bibitem[Lindley(1956)]{lindley1956measure}
Dennis~V Lindley.
\newblock On a measure of the information provided by an experiment.
\newblock \emph{The Annals of Mathematical Statistics}, pages 986--1005, 1956.

\bibitem[MacKay(2003)]{mackay2003information}
David~JC MacKay.
\newblock \emph{Information theory, inference and learning algorithms}.
\newblock Cambridge university press, 2003.

\bibitem[Shannon(1948)]{shannon1948mathematical}
Claude~Elwood Shannon.
\newblock A mathematical theory of communication.
\newblock \emph{The Bell system technical journal}, 27\penalty0 (3):\penalty0
  379--423, 1948.

\bibitem[Shwartz-Ziv and Tishby(2017)]{shwartz2017opening}
Ravid Shwartz-Ziv and Naftali Tishby.
\newblock Opening the black box of deep neural networks via information.
\newblock \emph{arXiv preprint arXiv:1703.00810}, 2017.

\bibitem[Williams(2011)]{williams2011information}
Paul~L Williams.
\newblock \emph{Information dynamics: Its theory and application to embodied
  cognitive systems}.
\newblock PhD thesis, PhD thesis, Indiana University, 2011.

\bibitem[Xu et~al.(2020)Xu, Zhao, Song, Stewart, and Ermon]{xu2020theory}
Yilun Xu, Shengjia Zhao, Jiaming Song, Russell Stewart, and Stefano Ermon.
\newblock A theory of usable information under computational constraints.
\newblock \emph{arXiv preprint arXiv:2002.10689}, 2020.

\bibitem[Yeung(1991)]{yeung1991outlook}
R.W. Yeung.
\newblock A new outlook on shannon's information measures.
\newblock \emph{IEEE Transactions on Information Theory}, 37\penalty0
  (3):\penalty0 466--474, 1991.
\newblock \doi{10.1109/18.79902}.

\bibitem[Yeung(2008)]{yeung2008information}
R.W. Yeung.
\newblock \emph{Information Theory and Network Coding}.
\newblock Information Technology: Transmission, Processing and Storage.
  Springer US, 2008.
\newblock ISBN 9780387792347.

\end{thebibliography}

\clearpage
\appendix

\section{On CoreSet-by-Disagreement}
\label{sec:csd_experiments}
\subsection{Experiments}

\textbf{MNIST.} We implement CSD and evaluate it on MNIST to show that it can identify a core-set of training samples that achieves high accuracy and low loss.

CSD is very sensitive to mislabeled samples because we compute the information gain using the provided labels: if a sample is mislabeled and the model has high confidence for the true label already, it will necessary have a very high information gain and the model will acquire this mislabeled sample.

To avoid this, we train a LeNet ensemble with 5 models on MNIST and discard all training samples with predictive entropy $> 0.01$ nats and whose labels do not match the predictions. This removes about 5678 samples from the training set.%

We use a LeNet model \citep{lecun1998gradient} with MC dropout (dropout rate $\tfrac{1}{2}$) in the core-set setting where we have access to labels but otherwise use an active learning setup. We use individual acquisition and compare to BALD, which does not make use of label information, and which we use as a sanity baseline. The training regime follows the one described in \citet{kirsch2019batchbald}. %

\begin{figure}[t]
  \centering
  \includegraphics[width=\linewidth]{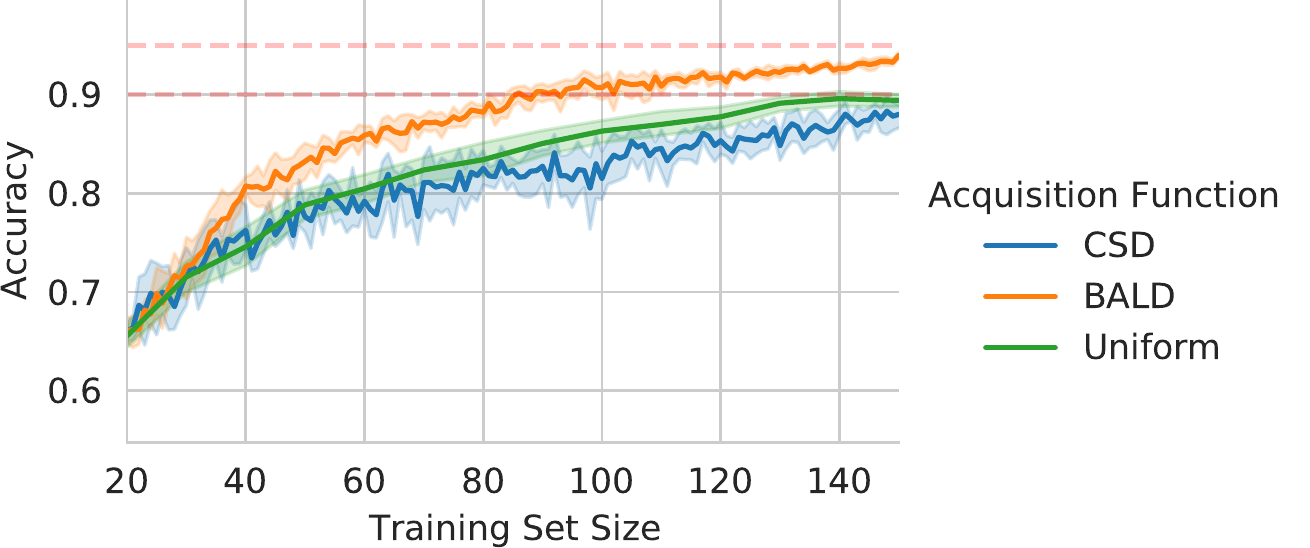}
  \caption{
    \emph{\emph{Ablation with ambiguous and mislabeled training samples included}: CSD vs BALD vs uniform acquisition on MNIST.} CSD performs worse than uniform acquisition. 5 trials each.
  }
  \label{fig:coresetbald_mnist_label}
\end{figure}

\begin{figure}[t]
  \centering
  \includegraphics[width=\linewidth]{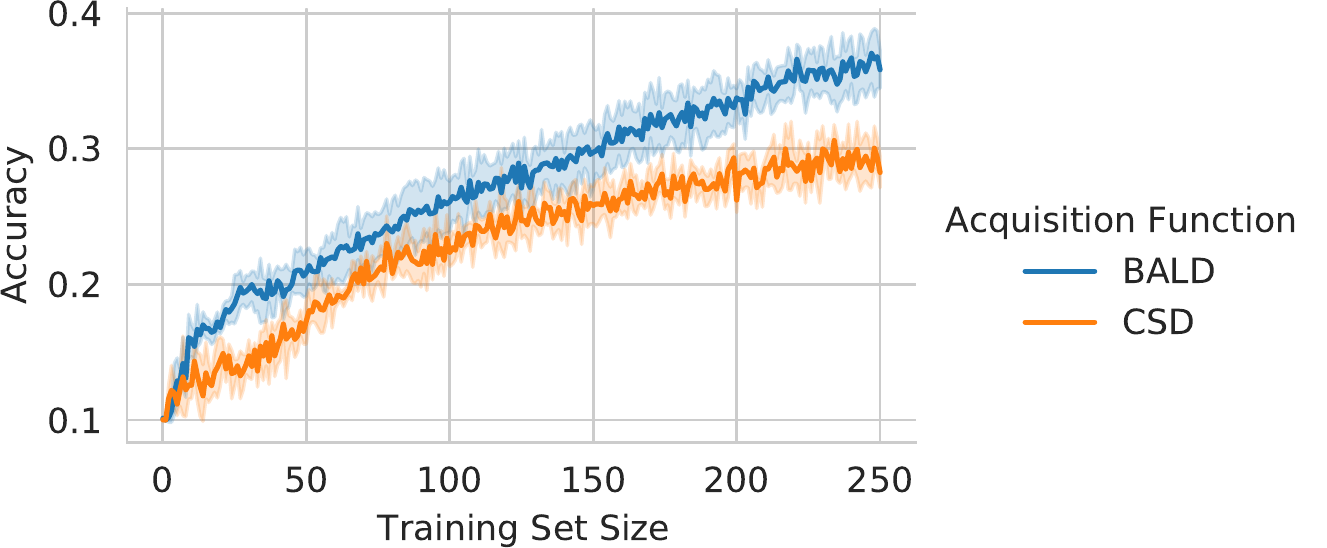}
  \caption{
    \emph{CSD vs BALD on CIFAR-10 without ambiguous and mislabeled training samples have been removed.} CSD performs worse than BALD. 5 trials each.
  }
  \label{fig:coresetbald_cifar10_clean_cold}
\end{figure}

\begin{figure}[t]
  \centering
  \includegraphics[width=\linewidth]{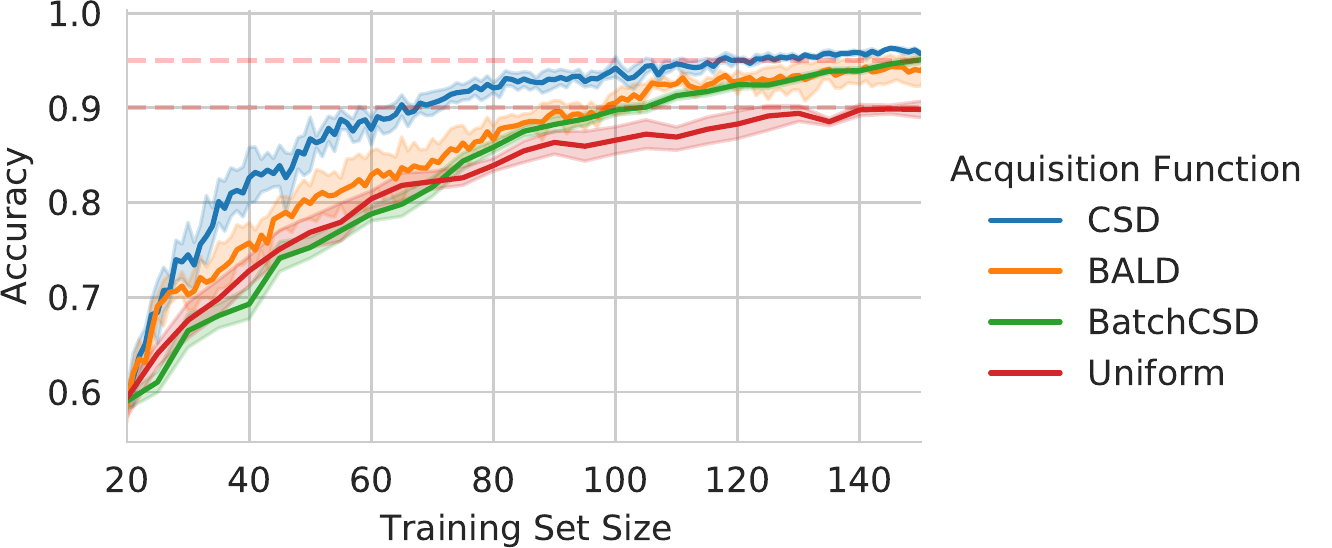}
  \caption{
    \emph{Ablation BatchCSD vs CSD (vs BALD vs uniform acquisition) on MNIST.} BatchCSD performs worse than BALD. 5 trials each.
  }
  \label{fig:coresetbald_mnist_clean_batch}
\end{figure}

\begin{table}[t]
  \centering
  \renewcommand{\arraystretch}{1.2}
  \caption{\emph{25\%/50\%/75\% quantiles for reaching 90\% and 95\% accuracy on MNIST.} 5 trials each.}
  \begin{tabular}{@{}ccc@{}}
  \toprule
  Acquisition Function & 90\% Acc    & 95\% Acc    \\ \midrule
  Uniform               & 125/130/150 & ---         \\
  BALD                 & 88/91/99    & 130/145/167 \\
  \textbf{CSD (ours)} & 55/\textbf{58}/58    & 105/\textbf{111}/115 \\
  \bottomrule
  \end{tabular}
  \label{tbl:mnist}
\end{table}

\Cref{fig:coresetbald_mnist} shows that CSD strongly outperforms BALD on MNIST (both with individual acquisitions). Indeed, only 58 samples are required to reach 90\% accuracy on average and 111 samples for 95\% accuracy compared to BALD which needs about 30 samples more in each case; see also \Cref{tbl:mnist}.

In \Cref{fig:coresetbald_mnist_label}, we show an ablation of using CSD without removing mislabeled or ambiguous samples from the training set. Here, BALD (without label information) outperforms CSD, which shows that CSD suffers from mislabeled examples.

\textbf{CIFAR-10.} However, we cannot produce the same results on cleaned CIFAR-10 (similar like MNIST described above) with ResNet18 models and MC dropout. BALD performs much better than CSD, even when cold starting. The accuracy plot is depicted in \Cref{fig:coresetbald_cifar10_clean_cold}. This indicates that something is wrong. We have not been able to identify the issue yet.

\textbf{BatchCSD.} Finally, we examine an extension of CSD to the batch case following \citet{kirsch2019batchbald} and compute $\MIof{\W ; \ybatchset \given \xbatchset}$ using the approximation $\MIof{\ybatchset ; \W \given \xbatchset}$. This approximation does not work well in the batch case, however,  even for a batch acquisition size of 5, as depicted in \cref{fig:coresetbald_mnist_clean_batch} (on MNIST). BatchCSD performs worse than Uniform for $\approx70$ samples and worse than BALD for 150 samples. A reason for this could be that the information gain and thus CSD are not submodular. This means that the sequential selection of acquisition (batch) samples has no optimality guarantee, unlike with BALD \citep{kirsch2019batchbald}.

\FloatBarrier

\subsection{Limitations of our Implementation \& Approach}
We have used our proposed notation to reinterpret BALD as the expected information gain and found an approximation for the information gain which allowed use to introduce CSD and show that it works on MNIST. But we have not been able to provide good results for CIFAR-10 or successfully extend our approximation to the batch case.
Moreover, the approximation we have used only works for MC dropout with dropout rate $\tfrac{1}{2}$. Our approach requires an explicit model, otherwise. Importantly, unlike BALD, the information gain in CSD does also not seem to be submodular, and we cannot infer a $1-\tfrac{1}{e}$ optimality that way \citep{kirsch2019batchbald}---although BALD's submodularity and optimality is not tied to the generalization loss anyway.

}
\end{document}